\documentclass[sigconf]{acmart}
\pdfoutput=1 % for arxiv
 
\renewcommand\footnotetextcopyrightpermission[1]{} % removes footnote with conference information in first column
\usepackage{algorithm}
\usepackage{algorithmic}
\usepackage[n,advantage, operators, sets, adversary, landau, probability, notions, ff, mm, primitives, events, complexity, oracles, asymptotics, keys]{cryptocode}
\usepackage{newfloat}
\usepackage{listings}
\floatstyle{ruled}
\newfloat{listing}{tb}{lst}{}
\floatname{listing}{Listing}
\usepackage{amsmath}

\usepackage{dsfont}
\usepackage{stmaryrd}
\usepackage{color}          % color
\usepackage{booktabs} % tabular
\usepackage{multirow} % tabular
\newcommand{\nosection}[1]{\vspace{1.5pt}\noindent\textbf{#1.}}

\usepackage{lineno}
\usepackage{subfigure}
\usepackage{enumitem}
\usepackage{setspace}
\AtBeginDocument{%
  \providecommand\BibTeX{{%
    \normalfont B\kern-0.5em{\scshape i\kern-0.25em b}\kern-0.8em\TeX}}}

\setcopyright{acmcopyright}
\copyrightyear{2018}
\acmYear{2018}
\acmDOI{XXXXXXX.XXXXXXX}

\acmConference[Conference acronym 'XX]{Make sure to enter the correct
  conference title from your rights confirmation emai}{June 03--05,
  2018}{Woodstock, NY}
\acmPrice{15.00}
\acmISBN{978-1-4503-XXXX-X/18/06}

\begin{document}

\title{Scalable and Sparsity-Aware Privacy-Preserving K-means Clustering with Application to Fraud Detection}

\author{Yingting Liu}
\affiliation{%
  \institution{Ant Group}
%   \city{Hangzhou}
  \city{Hangzhou}
  \country{China}
}
\email{yingting@antgroup.com}

\author{Chaochao Chen}
\authornote{Chaochao Chen is the corresponding author.}
\affiliation{%
  \institution{Zhejiang University}
  \city{Hangzhou}
  \country{China}
}
\email{zjuccc@zju.edu.cn}

\author{Jamie Cui}
\affiliation{%
  \institution{Ant Group}
  \city{Hangzhou}
  \country{China}
}
\email{Jamie.cui@outlook.com}

\author{Li Wang}
\affiliation{%
  \institution{Ant Group}
  \city{Hangzhou}
  \country{China}
}
\email{raymond.wangl@antgroup.com}

\author{Lei Wang}
\affiliation{%
  \institution{Ant Group}
  \city{Hangzhou}
  \country{China}
}
\email{shensi.wl@antgroup.com}

% \author{
%     Anonymous Author(s)
% }
% \institute{}

\renewcommand{\shortauthors}{Trovato and Tobin, et al.}

\begin{abstract}
% The abstract should briefly summarize the contents of the paper in
% 15--250 words.
K-means is one of the most widely used clustering models in practice.
%%%
Due to the problem of data isolation and the requirement for high
model performance, how to jointly build practical and secure K-means for multiple parties has become an important topic for many applications in the industry.
Existing work on this is mainly of two types.
%
% The first type has efficiency advantages, but information leakage raises potential privacy risks.
The first type has advantages in efficiency, but there is information leakage and potential privacy risks.
The second type is provable secure but is inefficient and even helpless for the large-scale data sparsity scenario. 
%%%
% In this paper, we propose a novel framework for efficient sparsity-aware K-means, which has three characteristics. % which is provable secure.
In this paper, we propose a new framework for efficient sparsity-aware K-means with three characteristics.
First, our framework is divided into a data-independent offline phase and a much faster online phase, and the offline phase allows to pre-compute almost all cryptographic operations. %, making the online phase as competitive as the plaintext training.
Second, we take advantage of the vectorization techniques in both online and offline phases.
Third, we adopt a sparse matrix multiplication for the data sparsity scenario to improve efficiency further.
%%%
%
We conduct comprehensive experiments on three synthetic datasets and deploy our model in a real-world fraud detection task. 
Our experimental results show that, compared with the state-of-the-art solution, our model achieves competitive performance in terms of both running time and communication size, especially on sparse datasets.
\end{abstract}


% \ccsdesc[100]{Networks~Network reliability}s


\settopmatter{printacmref=false}



\maketitle
\pagestyle{plain}

\section{Introduction}

    % 1. The importance of clustering. Applied into xxx. Most existing work focus on xxx, ignoring …
    %
    % As one of the most important algorithm in clustering, 
    K-means \cite{macqueen1967some} has been widely used in many applications, such as feature extraction \cite{kumari2019leaf}, document clustering \cite{sardar2018analysis}, targeted marketing \cite{kansal2018customer}, and outlier detection \cite{min2018k}.
    % image segmentation \cite{sharma2020advanced}, feature extraction \cite{kumari2019leaf}, and feature learning \cite{tang2017weed}. 
    Risk control, including fraud detection, is the core for financial companies.  
    %
    % Fraud detection is, therefore, a top priority in the business.
    %
    Since fraud patterns are diverse and transient and more than 90\% of them are unlabeled, traditional financial companies typically use the K-means algorithm for fraud detection, often in partnership with partner companies.
    For example, they group transactions based on the independent attribute values, such as credit card values from the payment company and buyer values from the merchant, which distinguish outliers when the input size is large enough \cite{chawla2013k}.
    However, over the past few years, the situation of \textit {isolated data islands} \cite{yang2019federated} has become a severe problem.
    Privacy laws and regulations are getting stricter, preventing data sharing between multiple entities. In contrast, a single company with limited data can hardly meet the high-precision requirements of modeling.
    %
    % Traditional K-means algorithms mainly focus on the selection of initial centers \cite{dong2009k} and the method of measuring similarity \cite{akay2018clustering}. But they ignore the issue of privacy protection, which can not be directly applied in some privacy-sensitive fields, such as finance and medical treatment.s
 
    % 2. Problem. Clustering under data isolation setting. Data sparsity problem. Emphasis the importance of security and efficiency
    \nosection{Application Scenarios}
    Solving the above dilemma in financial application scenarios needs to meet some requirements.
    First, financial scenarios are more strictly regulated than other industries, are more sensitive to information leakage, and prefer jointly trained models that reveal nothing but outputs.
    Secondly, due to the application in production, efficiency is also essential while meeting security.
    Lastly, sparseness is an inherent dataset property in real-world applications, mainly from incomplete user profiles or feature engineering such as one-hot.
    It causes a severe efficiency problem for existing privacy-preserving techniques \cite{chen2020homomorphic}.
    
    % 1. 金融场景的监管较其他行业严格，对于信息泄露比较敏感，更倾向于可证安全的技术。
    % 2. 其次，由于要实际应用于生产，对效率的要求也比较高，
    % 3. 实际应用中常有sparse
    %
    % In this paper, we mainly focus on developing the privacy-preserving K-means model, especially under the data-sparse scenario in the fraud detection task. 
    % %
    % That is, multiple participants jointly train K-means model using their private data without revealing any information, except the output. 
    % %
    % Specifically, our work can be adapted to two types of data distribution \cite{li2020federated}, i.e., \textit{vertically} partitioned data and \textit{horizontally} partitioned data. 
    % %
    % The former assumes that participants have the same batch of samples with different attributes, while the latter assumes that participant owns different samples but the same features. 
    % % 

    % 3. Existing work. Two types, have different shortcomings. Analysis Payman model and its problems. 
    \nosection{Existing Work and Unresolved Problems}
    The existing privacy-preserving K-means models mainly fall in two categories:(1) \textit{partial privacy protection} based and (2) \textit{provable security} based.
    The former attempts to use perturbation, permutation, or other cryptographic protocols to protect some sensitive information, such as labels and attributes.
    However, there is still a certain amount of information leakage in the model, such as cluster category \cite{xing2017mutual}, centroids \cite{vaidya2003privacy} and cluster to be merged \cite{jagannathan2010communication}.
    In contrast, the latter uses cryptographic protocols to design algorithms with provable security. 
    However, their running times are much longer than the former due to the high computation and communication complexity of the cryptographic techniques \cite{bunn2007secure,almutairi2017k,jaschke2018unsupervised}. 
    Payman Mohassel et al. \cite{mohassel2020practical} proposed the state-of-the-art privacy-preserving clustering scheme. 
    They proposed an ingenious distance computation protocol of centroids to all samples and built a customized garbled circuit to compute binary secret sharing of the minimum.
    Nevertheless, their scheme still does not work efficiently for three reasons.
    Firstly, they did not take advantage of the pre-computation property of cryptographic operations, which would otherwise result in a very efficient online phase \cite{demmler2015aby}.
    Secondly, they operated on numerical values, which is not as efficient as vectors in secret shared setting \cite{mohassel2017secureml}.
    Finally, they cannot handle the data sparsity scenario, which is common in practice.
    
% 4. Our solution. 
    \nosection{Our Proposal}
    To solve the above three problems and deploy fraud detection in financial scenarios, we design and implement an efficient privacy-preserving K-means clustering algorithm for the data-sparse scenarios in the semi-honest security setting. 
    % %
    % To solve the above three problems and deploy for fraud detection in financial scenario, we design and implement an efficient privacy-preserving K-means clustering algorithm for arbitrarily partitioned data in semi-honest security setting.  
    % %
    % In this paper, we mainly focus on developing the privacy-preserving K-means model, especially under the data-sparse scenario in the fraud detection task. 
    %
    % It can be adapted to two types of data distribution \cite{li2020federated}, i.e., \textit{vertically} partitioned data and \textit{horizontally} partitioned data. 
    %
    % The former assumes that participants have the same batch of samples with different attributes, while the latter assumes that participant owns different samples but the same features. 
    %
    % Multiple participants jointly train K-means model using their private data without revealing any information, except the output. 
    % 
    Each iteration of the K-means algorithm mainly consists of three steps, i.e., distance computation, cluster assignment, and centroids update. We optimize all of these steps using appropriate cryptographic techniques.
    First of all, inspired by prior work \cite{demmler2015aby}, all three steps can be divided into a data-independent offline phase and a much faster online phase. It allows to pre-compute almost all cryptographic operations offline, making the online phase as competitive as the plaintext training.
    Secondly, we take advantage of the vectorization techniques in both online and offline phases. Especially for distance calculation, most of the computations can be done using matrix instead of numerical operations, which can significantly reduce the number of interactions.
    Finally, under the data-sparse scenario, distance calculation and centroids update steps include many sparse matrix multiplications, which would be inefficient using pure secret sharing technique. A lot of unnecessary calculations come up because secret sharing splits 0 into random shares in finite space. That is, the sparse feature becomes dense after using secret sharing.
    By combining the characteristics of Secret Sharing (SS) \cite{karnin1983secret} and Homomorphic Encryption (HE) \cite{acar2018survey}, we design an efficient solution to improve the efficiency of privacy-preserving K-means clustering.
    The above steps can be further adapted to two types of data distribution \cite{li2020federated}, i.e., \textit{vertically} partitioned data and \textit{horizontally} partitioned data. 
    To this end, we can solve the three problems mentioned above.
    
    % \nosection{Results}
    % \ccc{Results}
    % For dense data scenario, the online time and communication of our framework is 5x faster than the state-of-the-art solution in average, while the overall running time and communication is comparable.
    % %
    % For sparse data scenario, results show that our  spare optimization can significantly improve efficiency. And the higher sparse degree is, the obvious improvement is.
    
    % \nosection{Results}
    % % TODO: Experimental result
    % We implement our algorithm on one public dataset and two synthetic datasets under the local area network or wide area network according to the need for comparison. The results demonstrate that: 
    % %
    % (1) The execution time of the online phase of our privacy-preserving k-means is 5x faster than the state-of-the-art solutions in average, i.e., the work proposed by Payman \cite{mohassel2020practical}. While the overall running time is about the same. 
    % %
    % (2) Our protocol outperforms the state-of-the-art in terms of online communication cost. Again, the overall communication cost of our protocol is the same.
    % %
    % (3) For the sparse dataset with the sparse degree of 20\%, when the feature dimension is 4, the running time of our scheme is roughly the same as the conventional SS scheme. As the feature dimension increases, the advantages of our solution become more and more obvious. When the feature dimension is 16, our shame is almost 1x faster than the base protocol.
    % %
    % It can be seen that we significantly improve the efficiency of the algorithm.
 
% 5. Contributions. 
    \nosection{Contributions}
    Our contributions can be summarized as follows: (1) We introduce an efficient and secure k-means protocol with the online-offline framework, which takes advantage of the vectorization techniques throughout the algorithm. (2) For the data-sparse scenario, we design a solution to improve the efficiency of the main steps in k-means by combining the characteristics of SS and HE. (3) We conduct the experiments for our proposed protocols and deploy them in a fraud detection task in the financial scenario. (4) The results show that our protocol significantly outperforms the state-of-the-art in the online phase.

\section{Related Work}
    % This section will mainly summarize the existing work on privacy-preserving k-means algorithms. All proposed works could be classified according to their privacy level into two categories: (1) partial privacy-protection based and (2) provable security based.
    
    % ===========
    \subsection{Partial Privacy Protection Based K-means}
    Algorithms in this category attempt to provide partial privacy protection for certain sensitive information. 
    The methods of Vaidya and Clifton \cite{vaidya2003privacy}, Jha et al. \cite{jha2005privacy}, and Jagannathan and Wright \cite{jagannathan2005privacy} provide security guarantees on attributes and labels of data from different parties. However, none of them can protect the intermediate centroids from being leaked. Centroids in each iteration are very sensitive because they reflect much sample information.
    Unlike these previous methods, the work of Jagannathan et al. \cite{jagannathan2010communication} does not reveal intermediate candidate cluster centers. However, it reveals the cluster category, which exposes which samples are qualitatively similar. Especially in the horizontal scene that they are working for, there is serious information leakage.
    In addition, there are other researches \cite{lin2011privacy,yi2013equally}, in which each iteration of k-means clustering can be performed without revealing most of the intermediate values. Nevertheless, they reveal the number of entities in each cluster, which also contains the general distribution information of private inputs.
    
    % ===========
    \subsection{Provable Security Based K-means}
    In contrast, the second category enjoys provable security using some cryptographic protocols such as oblivious transfer \cite{mohassel2020practical,DBLP:journals/joc/BunnO20}, homomorphic cryptosystems \cite{bunn2007secure,jaschke2018unsupervised} and secret sharing \cite{patel2012efficient,mohassel2020practical}. 
    They do not reveal any information such as the private inputs, cluster category, and numbers of entities in a cluster. However, most of them suffer from high communication and computation costs and cannot scale to large datasets. 
    One of the most notable examples in this category is  \cite{jaschke2018unsupervised}. They use fully homomorphic encryption and fractional encoding to allow addition, multiplication, and division on encrypted data. Unfortunately, it takes almost 1.5 years on a real-world dataset with 400 entries, 2 dimensions, and 3 classes, which is not practical for large data sets.
    % 
    % The scheme of Gheid and Challal \cite{gheid2016efficient} is only applicable to horizontally partitioned data.
    % outsourcing 
    % Besides, although the scheme implemented by outsourcing technology does not induce computation or communication overhead, it exposes the intermediate sensitive information to the server. However, in applications, it is difficult to obtain a trusted server. Without it, the security of this type of scheme is still difficult to guarantee.
    %
    Similar to Payman's privacy-preserving clustering scheme \cite{mohassel2020practical}, in this paper, we focus on how to propose new efficient privacy-preserving K-means clustering for multi-parties with full privacy guarantees to address three problems in their work.

\section{Preliminaries}
% This section first briefly introduces the preliminaries of cryptographic tools, including Secret Sharing(SS), Homomorphic Encryption(HE), and efficient conversion between them.
% % for sparse matrix multiplication.
% %
% Then we present the conventional K-means in plaintext.

% \subsection{Security Primitives}
%     According to our privacy requirements \ccc{where?}, each computed intermediate value will be keeping secret-shared. Participants can only receive the result (e.g. cluster centers) at the end of the protocol, and are invisible to others' inputs and the intermediate values during the entire execution of the algorithm. \ccc{wrong place?}
    
\subsection{Secret Sharing}
\label{Sec:Secret Sharing}
    Secret Sharing (SS) is a technique of secure Multi-Party Computation (MPC), which is widely used to build privacy-preserving machine learning models \cite{demmler2015aby,mohassel2017secureml,knott2021crypten,DBLP:conf/ijcai/0001ZZWLWWLWZ22}. In this paper, we use additive secret sharing \cite{singh2017secure}, which is an efficient scheme.
    
    % \noindent\textit{Sharing Semantics}
    \nosection{Sharing}
    Under additive secret sharing, a $l$-bit (e.g., $l \in \{1, 64, 128\}$) secret value $s$ will be additively split into $n$ shares and can be reconstructed by the sum of them. We describe the evaluation as follows. (1) \textit{Sharing}
        ${\langle x \rangle} \leftarrow \mathbf{Shr}_i(x)$: $P_i$ holds $x$, splits it into uniformly distributed random shares ${\langle x \rangle}$, and sends ${\langle x \rangle}_j$ to $P_j$, where ${\langle x \rangle}_j \in \mathds{Z}_{2^l} $ and $x = \sum_{j=1}^M {\langle x \rangle}_j $ mod $2^l$.
        When $l$ is not equal to 1, ${\langle x \rangle}$ is called Arithmetic share (A-share).
        Particularly, when $l=1$, $x$ is a binary value and ${\langle x \rangle}^B$ is called Boolean share (B-share).
        In other words, B-share can be seen as additive sharing in the field $\mathds{Z}_{2}$. 
        (2) \textit{Reconstruction}
        $\mathbf{Rec}_i(x)$: All parties send ${\langle x \rangle}$ to $P_i$ who computes $ x = \sum_{i=1}^M {\langle x \rangle}_i $ mod $2^l$.
    
    \nosection{{Operations}}
    Secretly shared values can perform homomorphic operations, e.g., \textbf{SADD} (addition), \textbf{SMUL} (multiplication) \cite{demmler2015aby}, and \textbf{CMP} (comparison) \cite{veugen2015secure}.
    Correspondingly, the SADD \& SMUL operations under B-share are implemented by XOR and AND,a respectively.
    For \textbf{CMP}, although the input is A-share, we convert it to B-share for comparison using \textbf{A2B}, \textbf{B2A}, and \textbf{MSB} (most significant bit). 
    \textbf{MUX} (multiplexers for conditional expressions) is used to pick out a specific value. 
    %
    % Because the conversion costs between A-share and B-share are so small that it even allows a full round of conversion for a single CMP operation \cite{demmler2015aby}.
    %
    % , \textbf{MSB} (most significant bit), and \textbf{MUX} (multiplexers for conditional expressions), 
    %
    % The notions of these operations are presented in Appendix A. \lyt{}
    %
    During a series of operations mentioned above, the intermediate values will be kept secret-shared, and the final result will be reconstructed only once at the end of the protocol. 
    The above secret sharing schema works in a finite field and can be extended to the real number field using fix-point representation \cite{li2018privpy,mohassel2017secureml}. 
    The notions are presented as follows:
    \begin{itemize}
        \item \textbf{SADD}: $\langle z \rangle \leftarrow \langle x \rangle + \langle y \rangle$, where $z = x + y$ mod $2^l$. It can be easily done by adding particitants' local shares and can be extended to linear operation that $\langle z \rangle \leftarrow \alpha  \langle x \rangle + \langle y \rangle + \beta $, where $z = \alpha x + y + \beta$ mod $2^l$.
        \item \textbf{SMUL}:
        $\langle z \rangle \leftarrow \mathbf{Mul}(\langle x \rangle + \langle y \rangle)$, where $z = x \cdot y$ mod $2^l$.
        \item \textbf{A2B}: $\langle x_0 \rangle^B, \cdots, \langle x_l-1 \rangle^B \leftarrow \mathbf{A2B}(\langle x \rangle) $, where $x = \sum_{i=0}^{l-1} x_i \cdot 2^i$.
        \item \textbf{B2A}: $(\langle x \rangle)\leftarrow \mathbf{A2B}(\langle x_0 \rangle^B, \cdots, \langle x_l-1 \rangle^B)$.
        \item \textbf{MSB}:$\langle x_0 \rangle^B \leftarrow \mathbf{MSB}(\langle x \rangle)$.
       \item \textbf{CMP}: $\langle z \rangle \leftarrow \mathbf{CMP}(\langle x \rangle, \langle y \rangle )$.
       \item \textbf{MUX}: $ \langle c \rangle = \mathbf{MUX}(\langle z \rangle, \langle x \rangle, \langle y \rangle) = \langle z \rangle * \langle x \rangle + (1 - \langle z \rangle) * \langle y \rangle )$.
    \end{itemize}

\subsection{Additive Homomorphic Encryption}
    Additive homomorphic encryption, such as Okamoto-Uchiyama (OU) \cite{okamoto1998new} and paillier \cite{paillier1999public}, is a method that supports secure addition when given a ciphertext. It mainly includes the following schema and operations. 
    %
    % The operation of ciphertext on the untrusted participant is secure and its result needs to be sent to the private key holder for decryption.
        
    \nosection{Encryption and Decryption}
    One party generates a public-private key pair $(pk, sk)$ and distributes $pk$ to the other party. A plaintext $u$ encrypted by $pk$ is denoted by $\llbracket u \rrbracket$, e.g., $\llbracket u \rrbracket = \mathbf{Enc}(pk; x, r)$, where $r$ is a random number that makes sure the ciphertexts are different in multiple encryptions even when the plaintexts are the same.
    Given a ciphertext $\llbracket v \rrbracket$, it needs to be decrypted with its corresponding private key, i.e., $v = \mathbf{Dec}(sk; \llbracket v \rrbracket)$.
    %
    % Similar to secret sharing, plaintext in this arithmetic must be the element of the finite field. 
    
    \nosection{Homomorphic Operation}
    We overload `$+$' as a homomorphic addition operation on ciphertext space.
    For any plaintext $u$ and $v$ encrypted by the same $pk$, additive homomorphic encryption satisfies $\llbracket u \rrbracket + \llbracket v \rrbracket = \llbracket u+v \rrbracket$. 
    %
    % Additive homomorphic encryption satisfies that for any plaintext $u$ and $v$ encrypted by the same $pk$, $\llbracket u \rrbracket + \llbracket v \rrbracket = \llbracket u+v \rrbracket$.
    %
    % \begin{equation}
    % \label{Eq:homo-1}
    % \llbracket u \rrbracket + \llbracket v \rrbracket = \llbracket u+v \rrbracket.
    % \end{equation}
    %
    There are two common variants of it, i.e., $\llbracket u \cdot v \rrbracket = u \cdot \llbracket v \rrbracket$ and $\llbracket u + v \rrbracket = u + \llbracket v \rrbracket$. 
    % \begin{equation}
    % \label{Eq:homo-2}
    % \llbracket u \cdot v \rrbracket = u \cdot \llbracket v \rrbracket,
    % \end{equation} 
    %
    % and 
    % \begin{equation}
    % \label{Eq:homo-3}
    % \llbracket u + v \rrbracket = x + \llbracket v \rrbracket,
    % \end{equation} 
    %
    % where the operations `$\cdot$' and `$+$' are overloaded as the corresponding operation on ciphertexts.
    
    % There are two common variants of homomorphic operations: (1) $ \llbracket u * v \rrbracket = u* \llbracket v \rrbracket $ and (2) $ \llbracket u + v \rrbracket = x + \llbracket v \rrbracket $, where the operations between plaintext and ciphertext ,i.e., `$*$' and `$+$', are overloaded as the corresponding operation on ciphertext space.

    % \nosection{Floating-Point Encoding and Vectorizing}
    % Similar to secret sharing, plaintext in this arithmetic must be the element of the finite field. In order to support floating-point operations and share convert, we adopt the same fix-point representing approach as secret sharing. Moreover, these operations can be extended to vectors and matrices component-wise. 
    
\subsection{Conversion Between SS and HE}

%  In the data sparse scenario, the operating efficiency can be further optimized. Inspired by \cite{chen2020homomorphic}, we use efficient conversions between SS and HE to handle the secure matrix multiplication of large-scale sparse data.
    
    % \nosection{Type Conversion}
    Conversions between SS and HE are 2-party functionalities.
    Here, we just introduce \textbf{HE2SS} protocol between parties $\mathcal{A}$ and $\mathcal{B}$, which converts data from HE format to SS format \cite{chen2020homomorphic}. 
    That is, $\langle X \rangle \leftarrow $ $\mathbf{HE2SS}(\llbracket X \rrbracket_B, \{ pk_B, sk_B \})$, where $\llbracket X \rrbracket_B$ is a ciphertext encrypted by party $\mathcal{B}$ and held by $\mathcal{A}$. 
    During implementation, $\mathcal{A}$ generates $\langle X \rangle_1$ using Pseudo-Random Generator (PRG) in the same field ($\mathds{Z}_{\phi}$) as HE, and sends $\llbracket \langle X \rangle_2 \rrbracket_b = \llbracket X \rrbracket_b - \langle X \rangle_1$ mod $\phi$ to $\mathcal{B}$. $\mathcal{B}$ decrypts it to get $\langle X \rangle_2$, where $ X = \langle X \rangle_1 + \langle X \rangle_2$ mod $\phi$. The correctness and security is guaranteed by $\llbracket u + v \rrbracket = u + \llbracket v \rrbracket$. 
    It can be applied to the multi-party setting as well.

\subsection{K-means Clustering Algorithm}

     As the most popular algorithm in clustering, K-means was first proposed by MacQueen \cite{macqueen1967some}. It can automatically partition a collection of data sets into separated groups according to the similarity of the objects, where users preset the number of clusters, and each cluster is described by its center. 
    There are many criteria to measure the similarity, e.g., Euclidean distance \cite{faisal2020comparative,na2010research}, manhattan distance \cite{suwanda2020analysis,faisal2020comparative}, and cosine similarity \cite{strehl2000impact,liu2013clustering}. In our case, we assume data objects are elements of $\mathds{R}^d$ and adopt the Euclidean distance as our criterion. 
    The K-means clustering algorithm consists of two phases: cluster centroids initialization and Lloyd's iteration. Its core step is Lloyd's iteration, which mainly consists of three parts, i.e.,\textit{distance compute}, \textit{cluster assignment}, and \textit{centroids update}. %, and \textit{boundary check}. 
    %
    % A brief description of K-means is presented in Appendix A.
    A brief description of K-means implementation is presented in Algorithm \ref{alg:K-means}.
    \begin{algorithm}
    \caption{K-means clustering}
    
    \textbf{Input}: Data $X_{(n \times d)}$, Clunster number $k$ $\qquad \qquad \quad \quad \quad$ \\
    \textbf{Output}: Centroids $\mu_{(k \times d)}$
    % \LinesNumbered
    \label{alg:K-means}
        \begin{algorithmic}[1]%行号
        \STATE {Randomly choose K centroids $\mu_0$} \\
        
            % \COMMENT{\textbf{\textbf{Begin Lloyd's step}}}
            \REPEAT
            
            \STATE{}
            \COMMENT{\textbf{\textbf{Distance Compute}}} \\
            \STATE Compute Euclidean distance $D_{(n \times k )}$, where $D_{ij} = ||X_i - \mu_j||_2$  
            
            \STATE{}
            \COMMENT{\textbf{\textbf{Cluster Assignment}}} \\
            \STATE Reassign the sample to their nearest center by comparison, gets binary matrix $C_{(n \times k)}$, $C_i =\underset{j}{\mathrm{argmin}}D_i$

            \STATE{}
            \COMMENT{\textbf{\textbf{Centroids Update}}} \\
            \STATE Recompute the cluster centers $\mu_j =\frac{\sum_{i=1}^{n} I\{c_i=j\} X_i}{\sum_{i=1}^{n} I\{c_i=j\}}$ 
            
            \UNTIL{Has repeated for a fix number of times or the improvement in one iteration is below a threshold}
            
            \RETURN {Centroids $\mu_{(k \times d)}$}
        \end{algorithmic}
    \end{algorithm}
%   ============================================
% \subsection{Problem Statement}
    % \nosection{Data Setting} 
    % Data provided by different participants are partitioned horizontally or vertically in the data space \cite{li2020federated}. The former assumes that participants have the same batch of samples with different attributes, while the latter assumes that participant owns different samples but the same features. Our work can be adapted to the two categories of data distribution.
    
    % \nosection{Input} 
    % The plaintext data are provided by $M$ $(M>1, M \in Z^+)$ participants. We use $\boldsymbol{X_i}$ ($1 \leq i \leq M$) to denote the corresponding plaintext data of the $i$-th participant $P_i$, where each column represents an attribute, and each row represents a sample. The feature dimension and sample dimension of the joint data are denoted as $d$ and $n$. And the feature dimension and sample dimension of participant $i$ are denoted as $d_i$ and $n_i$, respectively.
    % %
    % In vertically partitioned situation, the joint data can be formalized as $\boldsymbol{X} = \left (\begin{smallmatrix} \boldsymbol{X_1}, & \boldsymbol{X_2}, & \cdots, & \boldsymbol{X_M} \end{smallmatrix}\right)$ and $n=n_1=\cdots=n_M, d=\sum^M_{i=1} d_i$.
    % %
    % In horizontally partitioned situation, the joint data can be formalized as $\boldsymbol{X} = \left (\begin{smallmatrix} \boldsymbol{X_1}^T & \boldsymbol{X_2}^T & \cdots & \boldsymbol{X_M}^T \end{smallmatrix}\right)^T$. Here, $n=\sum^M_{i=1} n_i,d=d_1=\cdots=d_M$. 

\section{The Proposed Method}
% In this section, we first describe our matrixed online-offline privacy-preserving k-means clustering protocol under normal settings. Then we propose an efficient protocol for data-sparse scenario.
% This section proposes a two-party privacy-preserving K-means framework and specific optimizations for data-sparse scenarios.
\subsection{Problem Statement and Setting}

    \nosection{Problem Statement}
        Our privacy-preserving K-means follows the framework of the standard K-means algorithm.
        It must achieve comparable performance with the plaintext K-means while protecting data privacy.
        In the following, we consider by default the setting of (semi-honest) 2PC with reprocessing.
        % 
        % Our contributions can apply to the multi-party setting as well.
        For the convenience of description, in the following, we take two parties as an example, which is easy to apply to the multi-party setting.
        The input is the data owned by two parties and the number of clusters $k$.
        And the output is the share of the final cluster allocation held by each party.
        %

        % Privacy-preserving K-means algorithm has two main purposes, namely lossless and privacy-preserving ability.
        % %
        % For lossless, the performance of the algorithm must be comparable to the non-private solution that brings all data in one place while protecting the data privacy under semi-honesty setting.
        % % 
        % For the ability to protect privacy, the model needs to ensure that participants cannot receive any intermediate information except for the output at the end of the algorithm.

    \nosection{Data Setting}
        Data provided by different parties is partitioned horizontally or vertically in the data space \cite{li2020federated}. 
        %The former assumes that participants have the same batch of samples with different attributes, while the latter assumes that the party owns different samples but the same features. 
        %
        We use $X_A / X_B$ to denote the plaintext data of party $A/B$, where each column represents an attribute and each row represents a sample. 
        The feature dimension and sample size of party $A/B$ and joint data are denoted as $d_A/d_B/d$ and $n_A/n_B/n$ respectively. 
        % Correspondingly, the feature dimension and sample dimension of the joint data are denoted as $d$ and $n$. 
        %
        In horizontally partitioned situation, the joint data can be formalized as $ \mathcal{X} = \left [\begin{smallmatrix} \boldsymbol{X_A}^T & \boldsymbol{X_B}^T \end{smallmatrix}\right]^T $ and $n=n_A+n_B,d=d_A=d_B$. In vertically partitioned situation, the joint data can be formalized as $\mathcal{X} = \left[\begin{smallmatrix} \boldsymbol{X_A}, & \boldsymbol{X_B} \end{smallmatrix}\right] $, and $n=n_A=n_B,d=d_A+d_B$.
    
    \nosection{Security Setting} 
        We consider the standard semi-honest model, widely used in many work \cite{brickell2005privacy,li2018privpy,araki2016high,zhang2013cryptanalysis,DBLP:conf/ijcai/0001ZZWLWWLWZ22}, where a probabilistic polynomial-time adversary with honest-but-curious behaviors is considered. 
        Assume that the adversary engages in protocol strictly while trying to preserve all intermediate outcomes and infer as much as possible.
        We also assume that parties do not collude with each other.
           
        %s
        % We consider the standard semi-honest model, widely used in many work \cite{li2018privpy,mohassel2017secureml}, where a probabilistic polynomial-time adversary with semi-honest behaviors is considered. That is, the adversary is assumed to follow the protocol %specifically 
        % while attempts to obtain additional information about the input of the other party. We also assume that parties do not collude with each other.
           
        % In this paper, we assume that the adversary is honest-but-curious, which means the adversary will engage in the protocol strictly, but will keep all of the intermediate results and try to infer as much as possible. 

    \nosection{Online-Offline Setting}
        Similar to \cite{demmler2015aby,mohassel2017secureml,damgaard2013practical}, we split our protocols into a data-independent offline phase and a much faster online phase. The offline phase consists mainly of cryptographic operations, 
        % such as multiplication triplets generation \cite{beaver1991efficient}, 
        which can be performed without the presence of data.
	    Meanwhile, the online stage contains the data-dependent steps in the K-means algorithm.
        Take two-party secret sharing matrix multiplication of as an example \cite{mohassel2017secureml,beaver1991efficient}. To multiply two secretly shared matrices $\langle A \rangle$ and $\langle B \rangle$, firstly, we prepare a shared matrix triple (Beaver’s triplet) $\langle U \rangle$ $\langle V \rangle$ and $\langle Z \rangle$, where each element in $U$ and $V$ is uniformly random in $\mathds{Z}_{2^l}$ and $Z = U V$ mod $2^l$.
        Secondly, given two shared matrices $\langle A \rangle$ and $\langle B \rangle$, two parties can compute $\langle E \rangle = \langle A \rangle - \langle U \rangle$ and $ \langle F \rangle = \langle B \rangle - \langle V \rangle$ locally. After one round interaction, both parties can reconstruct $E$ and $F$ and get the final multiplication result $\langle C \rangle = -iEF + \langle A \rangle_i F + E \langle B \rangle_i + \langle Z \rangle_i$, where $i \in \{0, 1\}$.
        The triplets generation step (the first step) is time-consuming because it needs a large number of cryptographic operations, such as oblivious transfer or homomorphic encryption \cite{demmler2015aby,mohassel2017secureml}. 
        Fortunately, this step is data-independent and can be prepared in advance as an offline phase, using either cryptography-based methods or a trusted third party. 
        After it, the online phase (the second step), which depends on the input data, could be done efficiently.

        % In particular, we can let the clients generate the multiplication triplets. Since the clients need to secretly share their data in the first place, it is natural to further ask them to secretly share some extra multiplication triplets. Now, these multiplication triplets can be generated in a trusted way with no heavy cryptographic operations, which improves the efficiency significantly. However, despite its benefits, it changes the trust model and introduces some overhead for the online phase.
    
\subsection{Privacy-Preserving K-means}
    Recall that the K-means clustering can be divided into \textit{cluster centroids initialization} and \textit{Lloyd's iteration}. In this subsection, we will first introduce some initialization methods, then describe the vectored secure Lloyd's iteration, which helps improve the efficiency of the privacy-preserving K-means algorithm.
    
    \nosection{Initialization}
        Cluster centroids initialization is an important issue because it determines the convergence rate.
        It can be done using different strategies. 
        A simple and common strategy is random initialization. All parties can pick random values or jointly negotiate random indexes of all K groups.
        There is another way to start with better initials. Each party locally runs the plain-text K-means clustering on its own data first. 
        %
        % For vertically partitioned data, each party can generate $\mu^{(0)}_p$($p \in \{A,B\}$) with dimension of $k \times d_p$. All of the $\mu^{(0)}_p$ will be distributed to each other into shares, and can be concatenated together into $\langle \mu^{(0)} \rangle$ with dimension of $k \times d$.
        %
        % For horizontally partitioned data, each can generate $\mu^{(0)}_p$($p \in \{A,B\}$) with dimension of $k \times d$, secret share them, and calculate their average as $\langle \mu^{(0)} \rangle$.
        %归一化
        In addition, before performing clustering, a joint normalization operation is required.

    \nosection{Privacy-Preserving Vectorized Lloyd's Iteration}
    \label{sec:lloid}
        After initialization, K-means will iteratively update the centroid until convergence using Lloyd's iteration. 
        During each iteration, there are three steps %repeated for $T$ times or until the improvement in one iteration is below a threshold: 
        i.e., \textit{Secure Distance Computation}, \textit{Secure Cluster Assignment}, and \textit{Secure Centroid Update}. %, and \textit{Secure Boundary Check}. 
        In these steps, we focus on two issues: security and efficiency. 
        For security, once the input data is shared through SS, during the whole process, each calculated intermediate value (e.g., Euclidean squared distance) is secretly shared as two uniformly distributed values held by each party.
        The final result will be reconstructed only once at the end of the protocol.
        Meanwhile, we solve the efficiency problem in two ways.
        Firstly, we divide the secure computation steps into a \textit{data-independent offline phase} and a \textit{data-dependent online phase}. 
        % The offline phase contains mainly time-consuming cryptographic operations and can be performed in advance, which makes the online phase very efficient. 
        %
        Secondly, we use vectorization to speed up the protocols. Compared with privacy-preserving numerical operations, the matrix form significantly reduces computation and communication in both online and offline phases. 
        %
        % Vectorization makes the protocol significantly more efficient, which we will report in experiments. 
        %
        Below we will present the details of vectorized secure Lloyd's iteration.
        % First, the distributed data should be shared with each other using secret sharing.
        % The main operations involved in this step are addition, multiplication, and comparison. Other operations, such as division, can be converted into the aforementioned operations \cite{mohassel2017secureml,catrina2010secure}.
        
        \nosection{Secure Distance Computation $\mathcal{F}_{ESD}$}
            Secure distance computation aims to measure the similarity between all samples and the centroids using Euclidean Square Distance (ESD). 
            In order to simplify the calculation, we use Euclidean Squared Distance (\textbf{ESD}) as the substitute for the Euclidean Distance in the SS state.
            %
            % Formally, in this step, we need to calculate the distance from all sample points to all centroids securely.
            %
            Formally, take a sample $\mathcal{X}_i$ and a centroid $\mu_j^{(t)}$ of $j$-cluster at the round $t$ as an example. 
            Both parties can calculate their A-share of ESD $\langle D_i^{(t)} \rangle$ with $\mathcal{X}_i$ and A-share $\langle \mu_j^{(t)} \rangle$ using $\mathcal{F}_{ESD}$:
            %
            % The Euclidean squared distance between the point $\mathcal{X}_i$ and the centroid $\mu_j^{(t)}$ calculated by $\mathcal{F}_{ESD}$ is given as follows:
                \begin{equation}
                \begin{aligned}
                % \small 
                % \begin{split}
                    D^{(t)}_i &= \mathcal{F}_{ESD}(\mathcal{X}_i, \langle \mu_j^{(t)} \rangle) 
                    = ||\mathcal{X}_i - \langle \mu_j^{(t)} \rangle||_2 \\
                    &= \sum_{l=1}^{d} \mathcal{X}_{il}^2 + \sum_{l=1}^d \langle \mu_{jl}^{(t)} \rangle^2 - 2\sum_{l=1}^d \mathcal{X}_{il} \langle \mu_{jl}^{(t)} \rangle.
                \label{d}
                % \end{split}    
                \end{aligned}
                \end{equation}
                % \hspace{-10mm}
            It is well-known that K-means only needs to compare the distance between all the samples and different cluster centers. 
            Since $\sum_{l=1}^{d} \mathcal{X}_{il}^2$ remains unchanged with a fixed $i$, it can be omitted when calculating Equation (\ref{d}). 
            That is, only $\langle {D'}_i^{(t)} \rangle$ needs to be calculated: 
            % $
            %             \langle {D'}_i^{(t)} \rangle = 
            %         \mathcal{F}_{ESD}'( \mathcal{X}_i, \langle \mu_j^{(t)} \rangle) 
            %         = \sum_{l=1}^d \langle \mu_{jl}^{(t)} \rangle^2 -2\sum_{l=1}^d \mathcal{X}_{il} \langle \mu_{jl}^{(t)} \rangle.
            % $
            \begin{equation}
            \langle {D'}_i^{(t)} \rangle = 
                    \mathcal{F}_{ESD}'( \mathcal{X}_i, \langle \mu_j^{(t)} \rangle) 
                    = \sum_{l=1}^d \langle \mu_{jl}^{(t)} \rangle^2 -2\sum_{l=1}^d \mathcal{X}_{il} \langle \mu_{jl}^{(t)} \rangle.
            \end{equation}
            %===============
            %     \begin{align}
            %     \scriptsize
            %     \begin{split}
            %         \langle {D'}_i^{(t)} \rangle = 
            %         \mathcal{F}_{ESD}'( \mathcal{X}_i, \langle \mu_j^{(t)} \rangle) 
            %         = \sum_{l=1}^d \langle \mu_{jl}^{(t)} \rangle^2 -2\sum_{l=1}^d \mathcal{X}_{il} \langle \mu_{jl}^{(t)} \rangle.
            %     \label{di}
            %     \end{split}
            %     \end{align}
            % Vectorize Equation (\ref{di}) to get:
            Vectorizing it with matrix form to get:
                \begin{equation}
                % \small
                \label{Eq:Vectorized ESD}
                \begin{aligned}
                    \langle D'^{(t)} \rangle = 
                    \mathcal{F}_{ESD}' (\mathcal{X}, \langle \mu^{(t)} \rangle)
                    = \langle U \rangle - 2 \mathcal{X} \langle \mu^{(t)} \rangle^T,
                \end{aligned}
                \end{equation}
            where $\langle U \rangle = \mathbf{1}_{n\times 1} [ |\langle \mu_1^{(t)} \rangle|^2, |\langle \mu_2^{(t)} \rangle|^2, \cdots, |\langle \mu_k^{(t)} \rangle|^2]_{1\times k}$ with the dimension of $(n \times k)$. 
            %
            % $\langle U \rangle = \mathbf{1}_{n\times 1} [\langle  \mu_1^{(t)} \rangle \cdot \langle \mu_1^{(t)} \rangle, \langle  \mu_2^{(t)} \rangle \cdot \langle \mu_2^{(t)} \rangle, \cdots, \langle  \mu_k^{(t)} \rangle \cdot \langle \mu_k^{(t)} \rangle]$
            %
            The traditional method needs to calculate the distance from each sample to each center point one by one.
            The total number of interactions in each iteration is $nk$.
            After vectorizing in Equation (\ref{Eq:Vectorized ESD}), the first item $\langle U \rangle$ only needs to be calculated once before the iteration starts, and the second item $\mathcal{X} \langle \mu^{(t)} \rangle^T$ only needs to be interacted with once per iteration due to the matrix operation.
            Thus, it is constructive for improving efficiency in high-latency scenarios.

            Below we will separately describe the details under two data distribution scenarios.
            % vertically partitioned scenario=========================
            Under vertically partitioned scenario, $\mathcal{X}_{(n \times d)} = [X_{A},X_{B}]$
            % $ \mu_{(k \times d)}^{(t)} = \langle \mu^{(t)} \rangle_A + \langle \mu^{(t)} \rangle_B$, and $C_{(n \times k)} = \langle C \rangle_A + \langle C \rangle_B$.
            % , where $\langle  \mu^{(t)} \rangle_{A}, \langle C \rangle_A$ is held by participant A and $\langle \mu^{(t)} \rangle_{B}, \langle C \rangle_B$ held by B. 
            and we can get:
            %        
            % \begin{spacing}{0.5}
            \begin{small}
            \begin{equation}
                \begin{aligned}
            		\langle {D'}^{(t)} \rangle 
            		&=\mathcal{F}_{ESD}'(\mathcal{X}, \langle \mu^{(t)} \rangle) 
            		= \langle U \rangle -2 \mathcal{X} \langle \mu^{(t)} \rangle^T \\
            		&= \langle U \rangle - 2 [X_A(\langle \mu^{(t)} \rangle_{A1} + \langle \mu^{(t)} \rangle_{B1}) + X_B (\langle \mu^{(t)} \rangle_{A2} + \langle \mu^{(t)} \rangle_{B2})],
                \end{aligned}
            \end{equation}
            \end{small}
            % \end{spacing}
             		%%%%%
            % 		\begin{equation}
            %         \small  
            %         \begin{aligned}
            %     		&\langle {D'}^{(t)} \rangle =\mathcal{F}_{ESD}'(\mathcal{X}, \langle \mu^{(t)} \rangle) = \langle U \rangle -2 \mathcal{X} \langle \mu^{(t)} \rangle^T \\
            %     		&= \langle U \rangle - [X_A(\langle \mu^{(t)} \rangle_{A1} + \langle \mu^{(t)} \rangle_{B1})
            %     		+ X_B (\langle \mu^{(t)} \rangle_{A2} + \langle \mu^{(t)} \rangle_{B2})],
            %         \end{aligned}
            %         \end{equation}
    		where $\langle  \mu^{(t)} \rangle_{A1}$ with dimension $(k \times d_A)$ and $ \langle \mu^{(t)} \rangle_{A2}$  with dimension $(k \times d_B)$ are held by party A ($[\langle  \mu^{(t)} \rangle_{A1}, \langle  \mu^{(t)} \rangle_{A2}] = \langle  \mu^{(t)} \rangle_{A}$).
    		Similarly, $\langle  \mu^{(t)} \rangle_{B1}$ with dimension $(k \times d_A)$ and $ \langle \mu^{(t)} \rangle_{B2}$  with dimension $(k \times d_B)$ are held by party B ($[\langle  \mu^{(t)} \rangle_{B1}, \langle  \mu^{(t)} \rangle_{B2}] = \langle  \mu^{(t)} \rangle_{B}$).
            %		
            % Same with $\langle \mu^{(t)} \rangle_{B}$ ($[\langle  \mu^{(t)} \rangle_{B1}, \langle  \mu^{(t)} \rangle_{B2}] = \langle  \mu^{(t)} \rangle_{B}$).
            % 		And $\langle  \mu^{(t)} \rangle_{B1}$ with dimension $(k \times d_A)$ and $ \langle \mu^{(t)} \rangle_{B2}$ with dimension $(k \times d_B)$ are held by participant B.
    		%
    		Note that $X_A \langle \mu^{(t)} \rangle_{A1}^T$ and $X_B \langle \mu^{(t)} \rangle_{B2}^T$ can be computed locally by A and B respectively.
            While, $X_B \langle \mu^{(t)} \rangle_{A2}^T$ and $X_A \langle \mu^{(t)} \rangle_{B1}^T$ should be jointly computed using vectorized secret sharing multiplication.
            % horizontally partitioned data setting
            Under horizontally partitioned data setting, $\mathcal{X}_{(n \times d)} = \left[ \begin{array}{c} X_{A(n_A \times d)} \\ X_{B(n_B \times d)} \end{array} \right]$ and we can get:
            % \begin{spacing}{0.5}
            % \begin{small}
            \begin{equation}
                % \large
                % \setlength{\arraycolsep}{5pt}
                \begin{aligned}
                    \langle {D'}^{(t)} \rangle
                    & = \mathcal{F}_{ESD}'(\mathcal{X}, \langle \mu^{(t)} \rangle) = \langle U \rangle - 2 \mathcal{X} \langle \mu^{(t)} \rangle^T \\
                    & = \langle U \rangle - 2 \left[\begin{array}{c} X_A \\ X_B \end{array}  \right](\langle \mu^{(t)} \rangle _A^T + \langle \mu^{(t)} \rangle _B^T) \\
                    &= \langle U \rangle -  2 \left[ \begin{array}{c} X_A(\langle \mu^{(t)} \rangle _A^T + \langle \mu^{(t)} \rangle _B^T) \\ X_B(\langle \mu^{(t)} \rangle _A^T + \langle \mu^{(t)} \rangle _B^T) \end{array} \right], 
                \end{aligned}
            \end{equation}
            % \end{small}
            % \end{spacing}
            where $X_A\langle \mu^{(t)} \rangle _B^T$ and $X_B\langle \mu^{(t)} \rangle _A^T$ should be jointly computed using vectorized secret sharing multiplication.
            
        \nosection{Secure Cluster Assignment $\mathcal{F}^k_{min}$}
            This step aims to find the clusters to which each sample belongs. 
            For distances from one sample to all current centroids, we perform secure comparisons to find the minimum among them and securely mark its position as the clustered index. 
            The output of $\mathcal{F}^k_{min}$ is an A-share of the clustered index to which this sample belongs. 
            Formally, two parties run $ \langle C_i^{(t)} \rangle \leftarrow \mathcal{F}^k_{min}(\langle D_i^{(t)} \rangle) ( i \in [0,n])$ to reassign sample $i$ to its nearest center $C_i^{(t)} :=(0, \cdots , 0,1,0, \cdots, 0)$, where $1$ appears at the $j$-th element if cluster center $j$ is the closest to sample $i$.
            To accelerate function $\mathcal{F}^k_{min}$, we utilize the classical binary tree reduction method, which consists of $k-1$ CoMParison Module (CMPM). 
            Take the cluster assignment in Figure \ref{cmp_fig} as an example, where we assume the number of clusters $k$ is 6.
            All the nodes (shown in yellow rectangles) in this inversed tree structure consist of two parts. 
            The left part of a node (shown in red font) is the record of the smallest distance among the leaf nodes of its subtree. 
            In Figure \ref{cmp_fig}, $\langle 1 \rangle$ in the right node of layer 2 is the smallest distance in $\langle 7 \rangle, \langle 2 \rangle, \langle 1 \rangle, \langle 3 \rangle$.
            The right part of a node  (shown in black font) is the relative position of the smallest distance. In particular, the relative positions of all the leaf nodes in layer 4 are initialized to $\langle 1 \rangle$.
            The key to $\mathcal{F}^k_{min}$ is CMPM, and we describe the main idea of which in a dotted box in Figure \ref{cmp_fig}.
            The input of CMPM are two distances to be compared (e.g., $ D_{[0]}^{(t)}=\langle 2 \rangle, D_{[1]}^{(t)}=\langle 1 \rangle$) and their relative positions (e.g., $c_{[0]}^{(t)}=\langle 01 \rangle, c_{[1]}^{(t)}=\langle 10 \rangle $). 
            The CMPM mainly contains two primitive operations: \textbf{CMP} and \textbf{MUX}, as described in Section \ref{Sec:Secret Sharing}.
            First, we use \textbf{CMP} to compare ``less than'', which judges the shared most signification bit of the subtraction using the \textbf{A2B}, \textbf{B2A}, and \textbf{MSB}. 
            Take the details of a CMPM in Figure \ref{cmp_fig} for example. 
            When we need to compare distances $\langle 2 \rangle$ and $\langle 1 \rangle$, we extract and judge the shared sign bit of $\langle 2 - 1 \rangle$ to get the result $z = \langle 0 \rangle$ . 
            Second, we use \textbf{MUX} to pick out the smaller one from the distances to be compared. In Figure \ref{cmp_fig}, we get the smaller distance ($\langle 1 \rangle$) from two distances of child nodes ($\langle 2 \rangle, \langle 1 \rangle$). 
            Then the relative position of the smaller distance can be securely marked by concatenation of $\langle b \rangle \langle c_{[0]}^{(t)} \rangle$ and $ (\langle 1 \rangle - \langle b \rangle) \langle c_{[1]}^{(t)} \rangle$, e.g., $ \langle 0010 \rangle \leftarrow \mathcal{F}^4_{min}(\langle 7 \rangle, \langle 2 \rangle, \langle 1 \rangle, \langle 3 \rangle)$ in Figure \ref{cmp_fig}.
            %
            % As mentioned above the output of a CMPM are the smaller one of distances and the secret shared index vector $\langle c \rangle$ as the output of $ \langle C_i^{(t)} \rangle \leftarrow \mathcal{F}^4_{min}(\langle D_1^{(t)} \rangle, \langle D_2^{(t)} \rangle, \langle D_3^{(t)} \rangle, \langle D_4^{(t)} \rangle)$.
            %
            Finally, after the reduction of the whole tree, we can get secretly shared index matrix $ \langle C_i^{(t)} \rangle$ of sample $i$ as the output of  $\mathcal{F}^k_{min}(\langle D_i^{(t)} \rangle) ( i \in [0,n])$ .
            %
            % The output of the CMPM at root node is the A-share of the whole index vector during the path. While other nodes also output the minimum value on the path.
            % `1' appears in the coordinate where the smallest number duing the path is.
            % $(Z_1 - Z_2) \leq 0$ implies $Z_1 \leq Z_2$ else otherwise, 
            % This can then be repeated K − 1 times to find the minimum of K numbers. After k−1 such comparisons, keeping the minimum each time, the minimum cluster is known.
            %
            % \begin{figure}[htbp] % H为当前位置，!htb为忽略美学标准，htbp为浮动图形
            %     \centering % 图片居中
            %     \includegraphics[width=0\textwidth]{compareM_adjust.pdf}
            %     \caption{Main name 2} 
            %     \label{Fig.main2} 
            % \end{figure}
            %
        %     \begin{figure*}[t]
        %     \centering
        %     \includegraphics[width=13.5cm]
        %     % \vskip -0.15in
        % % 	\center{\includegraphics[width=\columnwidth]}  
        % 	{figure/compareM.pdf}
        % % 	\vskip -0.4in
        %     	\caption{An example to illustrate the implementation of the secure cluster assignment protocol $\mathcal{F}^k_{min}$, where we assume the number of clusters $k$ is 6.}
        %     	\label{cmp_fig}
        %     \end{figure*}
        % %  	\vskip -0.2in

            \begin{figure}
            	\center{\includegraphics[width=\columnwidth]  {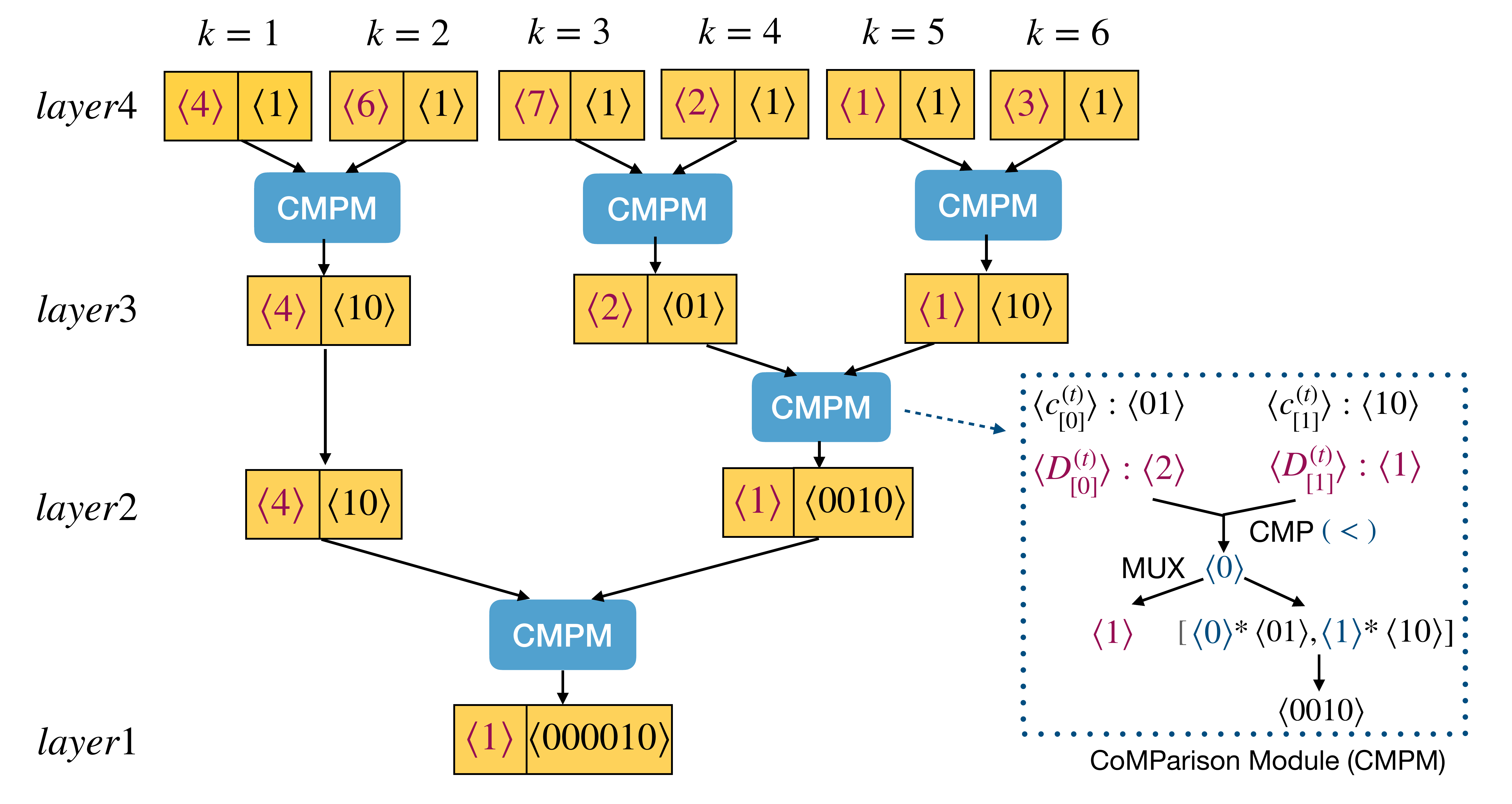}} 
            	\caption{An example to illustrate the implementation of the secure cluster assignment protocol $\mathcal{F}^k_{min}$, where we assume the number of clusters $k$ is 6.}
            	\label{cmp_fig}
            	\vskip -0.15in
            \end{figure}

        \nosection{Secure Centroids Update $\mathcal{F}_{SCU}$}
            After reassigning the sample to their nearest center, the cluster centers $\mu^{(t+1)}$ should be recomputed by $\langle C^{(t)} \rangle$ and $\mathcal{X}_i$ below:
                \begin{equation}
                \label{Eq:SCU}
                \begin{aligned}
                    \langle \mu^{(t+1)} \rangle = \mathcal{F}_{SCU}(\langle C^{(t)} \rangle, \mathcal{X}) 
                    =\frac{\sum_{i=1}^{n} \langle I^{(t)}\{c_i=j\} \rangle \mathcal{X}_i}{\sum_{i=1}^{n} \langle I^{(t)}\{c_i=j\} \rangle} 
                    = \frac{\langle C^{(t)} \rangle ^T \mathcal{X}}{\mathbf{1}_{1\times n} \langle C^{(t)} \rangle},
                \end{aligned}
                \end{equation}
            where $1 \leq j \leq k, 1 \leq i \leq n$, and $I\{c_i=j\}$ is a indicator function that equals 1 if $c_i = j$ and 0 otherwise.
            To compute the secret sharing of the updated cluster $\langle \mu^{(t+1)} \rangle$, parties separately compute the numerator and denominator.
            For vertically partitioned scenario, $\langle C_j \rangle ^T \mathcal{X} = (\langle C\rangle _A + \langle C\rangle _B)_j^T [X_{A{(n \times d_A)}},X_{B(n \times d_B)}]$. And for horizontally partitioned data setting, $\langle C_j \rangle ^T \mathcal{X} = (\langle C\rangle _A + \langle C\rangle _B)_j^T \left[ \begin{array}{c} X_{A(n_1 \times d)} \\ X_{B(n_2 \times d)} \end{array} \right]$. 
            Their processing is the same as the distance computation.
            If the inputs are local, the multiplication can also be calculated locally.
            Other multiplications are computed jointly.
            Then the reminders are calculated using the broadcasting secure division operation, which is converted to secure multiplication and addition.
            
        \nosection{Checking the Stopping Criterion  $\mathcal{F}_{CSC}$}
            The iteration terminates when it converges or reaches a certain iteration. For convergence, parties can jointly invoke the secure comparison protocol $\mathcal{F}_{CSC}(\langle \mu^{t} \rangle, \langle \mu^{t+1} \rangle, \epsilon) \leftarrow \mathbf{CMP} (\mathcal{F}_{ESD}(\langle \mu^{t} \rangle, \langle \mu^{t+1} \rangle), \epsilon)$ to check the stop threshold $\epsilon$.
        
        %=====================================================
        % A brief description of implementation of privacy-preserving vectorizated Lloyed's step is presented in Algorithm \ref{alg:vec_lloyd}. 
          
        %     \begin{algorithm}
        %     \caption{Privacy-Preserving Vectorizated Lloyed's step}
        %     \hspace*{0.0in} {\bf Input:} Data $X_{(n \times d)}$ and  Centroids of Initialization $\langle \mu_t \rangle$  \\
        %     \hspace*{0.0in} {\bf Output:} Centroids $\langle \mu \rangle$ 
        %     % \LinesNumbered
        %     \label{alg:vec_lloyd}
        %         \begin{algorithmic}[1]%行号
        %         \REPEAT
                
        %         \STATE Compute Euclidean distance $\langle D' \rangle$, where $ \langle D'_{ij} \rangle = \langle U \rangle -2X \langle \mu_t \rangle^T $  \COMMENT{\textbf{\textbf{Distance Compute}}}
                
        %         \STATE Reassign the sample to their nearest center by comparison, gets binary matrix $\langle C_t \rangle $, where $\langle C_i \rangle =\underset{j}{\mathrm{argmin}}D_i$ \COMMENT{\textbf{\textbf{Cluster Assignment}}}
        
        %         \STATE Recompute the cluster centers $\mu_t =\frac{\langle C_t \rangle^TX}{\mathbf{1}_{1\times n}\langle C_t \rangle}$ \COMMENT{\textbf{\textbf{Centroids Update}}}
                
        %         \UNTIL{Has repeated for $t$ times}
                
        %         \RETURN {Centroids $\langle \mu \rangle$}
        %         \end{algorithmic}
        %     \end{algorithm}

\subsection{Optimization for Sparse Scenario} 
    % In this section, we first describe the motivation of our optimization and then present a sparse privacy-preserving K-means algorithm.
    
    \nosection{Motivation}
        % Data-Sparse Setting
        \textit{Feature sparsity} is usually caused by missing feature values or feature engineering such as one-hot.
        Such data has two characteristics. 
        First, it will become dense after using secret sharing.
        % which will lead to a lot of unnecessary calculations in secure multiplication. 
        %
        For example, there is a 4-dimensional sparse vector $(0,0,1,0)$. After secret sharing, it will be randomly split into two shares uniformly distributed over a finite field, e.g., $(13,1,7,8)$ and $(3, 15, 10, 8)$ in $\mathds{Z}_{2^4}$.
        This will cause much unnecessary computation since we cannot determine the position of 0 in the state of sharing. Especially in the high-dimensional sparse feature situation, communication becomes the bottleneck of the SS-based model. It can even make the model unusable in practice.
        Second, it is usually not all 0s for the entire row or column.
        Therefore, our algorithm should be optimized for the scenario where any element in the data may be 0.

\nosection{Privacy-Preserving Sparse K-means}
        After vectorization, we can see that the cluster update involves many matrix multiplications. 
        There is a secure sparse matrix multiplication protocol combining the advantages of HE and SS \cite{chen2020homomorphic}. 
        Given a sparse matrix $X$ held by $\mathcal{A}$ and a dense matrix $Y$ held by $\mathcal{B}$, 
        the protocol can effectively calculate $X Y$ without revealing the value of $X$ and $Y$, which multiplies sparse data under HE space and uses HE2SS to convert the output to the secret-shared state for facilitating the following calculations under SS space.
        %
        % We present the pseudo-code of this protocol in Appendix B.
        %
        %
    % \begin{algorithm}
    %     \floatname{algorithm}{Protocol} 
    %     \caption{Secure Sparse Matrix Multiplication}
    %     \label{pro:SparseMul}
    %     \textbf{Input}: sparse matrix X held by A, matrix Y held by B $\qquad \qquad$ \\
    %     % HE key pair of B ($\{pk_b, sk_b$\})\\
    %     \textbf{Output}: $\langle Z \rangle_1$ for A and $\langle Z \rangle_2$ for B, where $Z = XY$  $\qquad \qquad \quad $ 
        %
    %         \begin{algorithmic}[1]%行号
    %         \STATE B generates HE key pair ($\{pk_b, sk_b$\}), encrypts Y with $pk_b$, and sends $\llbracket Y \rrbracket_b$ to A
    %         \STATE A calculates $\llbracket Z \rrbracket_b = X\llbracket Y \rrbracket_b$
    %         \STATE A locally generates share $\langle Z \rangle_1$ from $\mathds{Z}_{2^l}$
    %         \STATE A calculates $\llbracket \langle Z \rangle_2 \rrbracket_b = \llbracket Z \rrbracket_b - \langle Z \rangle$ mod $\mathds{Z}_{2^l}$ and sends $\llbracket \langle Z \rangle_2 \rrbracket_b$ to B 
    %         \STATE B decrypts $\llbracket \langle Z \rangle_2 \rrbracket_b$ to get $\langle Z \rangle_2$
    %         \RETURN {$\langle Z \rangle_1$ for A and $\langle Z \rangle_2$ for B, where $Z = XY$}
    %         \end{algorithmic}
    %     \end{algorithm}
        % \vskip -0.5in
        %
        As shown in Protocol \ref{pro:SparseMul}, 
        % given a sparse matrix X hold by participant A and a dense matrix Y hold by B, it can securely calculate $XY$ without revealing the value of X and Y.
        %
        Line 2 is the ciphertext multiplication using additive HE.
        It can significantly speed up sparse matrix multiplication by eliminating calculations involving zeros because the sparse matrix is $X$ held by $\mathcal{A}$.
        Lines 3-5 show how to generate secret shares under homomorphically encrypted field.
        Compared with SS, this protocol does not need to transmit X-sized matrices.
        Therefore, the communication cost of it will be much cheaper when the shape of Y is much smaller than X.
        %
        % We present the pseudo-code of this protocol in Protocol \ref{pro:SparseMul}.
        % %
        % given a sparse matrix X hold by participant A and a dense matrix Y hold by B, it can securely calculate $XY$ without revealing the value of X and Y.
        % %
        % Line 2 is the ciphertext multiplication using additive HE, which can significantly speed up sparse matrix multiplication by eliminating calculations involving zeros. Lines 3-5 show how to generate secret shares under the homomorphically encrypted field.
        % %
        % In addition, the communication cost of this protocol in K-means will be much cheaper because Y is smaller than X.
        %
        \begin{algorithm}[t]
        \floatname{algorithm}{Protocol} 
        \caption{Secure Sparse Matrix Multiplication}
        \label{pro:SparseMul}
        \textbf{Input}: A sparse matrix X hold by A, a matrix Y hold by B, HE key pair of B ($\{pk_b, sk_b$\})\\
        \textbf{Output}: $\langle Z \rangle_1$ for A and $\langle Z \rangle_2$ for B, where $Z = XY$
        
            \begin{algorithmic}[1]%行号
            \STATE B encrypts Y with $pk_b$ and sends $\llbracket Y \rrbracket_b$ to A
            \STATE A calculates $\llbracket Z \rrbracket_b = X\llbracket Y \rrbracket_b$
            \STATE A locally generates share $\langle Z \rangle_1$ from $\mathds{Z}_{2^l}$
            \STATE A calculates $\llbracket \langle Z \rangle_2 \rrbracket_b = \llbracket Z \rrbracket_b - \langle Z \rangle$ mod $\mathds{Z}_{2^l}$ and sends it to B 
            \STATE B decrypts $\llbracket \langle Z \rangle_2 \rrbracket_b$ to get $\langle Z \rangle_2$
            \RETURN {$\langle Z \rangle_1$ for A and $\langle Z \rangle_2$ for B, where $Z = XY$}
            \end{algorithmic}
        \end{algorithm}
        % %
        % This protocol can deal with the sparse feature cases of high dimension and zero at any position while ensuring the security of the algorithm.\lyt{Changed}
        %
        %==================
        We show the overall framework of our privacy-preserving K-means algorithm with sparse optimization in Algorithm \ref{alg:ver_sparse},
        where we take the vertical distribution of data as an example. 
        For sparse optimization, we replace the secure multiplication in lines 6 and 12 with the secure sparse multiplication mentioned above. 
        Taking the  $ X_A \langle \mu^{(t)} \rangle_{B1(k \times d_A)}^T $ in line 6 as an example,
        party B encrypts the $\langle \mu^{(t)} \rangle_{B1(k \times d_A)}^T $ and sent it to party A.
        Party A picks out the non-zero elements of $X_A$ and calculates them with $ \llbracket \langle \mu^{(t)} \rangle_{B1(k \times d_A)}^T \rrbracket $ based on HE operation.
        Note that the shape of $X_A$ is related to the sample size and it is much larger than shape of $ \llbracket \langle \mu^{(t)} \rangle_{B1(k \times d_A)}^T \rrbracket$.
        Therefore our framework can handle the high-dimensional sparse feature situation and have good communication efficiency. 
        Although it requires more computation, it is suitable for most practical applications, such as large companies are often computationally resource-rich but bandwidth-constrained.
        \begin{algorithm}[tb]
            \caption{Privacy-preserving  (sparse)  K-means for vertically partitioned data}
            \textbf{Input}: sparse matrix $X_A$, $X_B$ hold by A, B respectively (the joint data can be formalized as $\mathcal{X}= [X_A,X_B]$ and $n=n_A+n_B,d=d_A=d_B$); cluster number $k$ \\
            % \textbf{Parameter}: cluster numbers $k$\\
            \textbf{Output}: Centroids $\langle \mu^{(t)} \rangle$
            % \LinesNumbered
            \label{alg:ver_sparse}
            \begin{algorithmic}[1]%行号
            % \STATE {Randomly choose K centroids $\mu^{(0)}$ or each of them locally run the K-means to get $\mu^{(0)}_p$($p \in \{A,B\}$) with dimension of $k \times d_p$. All of the $\mu^{(0)}_p$ will be distributed to each other into shares, and can be concatenated together into $\langle \mu^{(0)} \rangle$ with dimension of $k \times d$.} \\
            \STATE {Randomly choose K centroids $\mu^{(0)}$ or each of them locally run the K-means for initialization.} \\
            \REPEAT
                
                % \STATE Compute Euclidean distance $\langle D' \rangle = U - ( X_A \langle \mu\rangle _{A1(k \times d_A)}^T + X_B \langle \mu\rangle _{A2(k \times d_B)}^T + X_A \langle \mu\rangle _{B1(k \times d_A)}^T + X_B \langle \mu\rangle _{B2(k \times d_B)}^T) $, where $X_B \langle \mu\rangle _{A2(k \times d_B)}^T$ and $X_A \langle \mu\rangle _{B1(k \times d_A)}^T$ should be jointly computed using secret sharing sparse multiplication.
                
                \STATE{}
                \COMMENT{\textbf{\textbf{Distance Compute}}} \\
                \STATE {A locally calculates $X_A \langle \mu^{(t)} \rangle_{A1(k \times d_A)}^T$, $\langle U \rangle_A$.}
                \STATE {B locally calculates $X_B \langle \mu^{(t)} \rangle_{B2(k \times d_B)}^T$, $\langle U \rangle_B$.}
                \STATE {A and B securely calculate $X_B \langle \mu^{(t)} \rangle_{A2(k \times d_B)}^T$ and $ X_A \langle \mu^{(t)} \rangle_{B1(k \times d_A)}^T $ using secure sparse matrix multiplication to get $\langle D' \rangle $.}
                
                \STATE {}
                \COMMENT{\textbf{\textbf{Cluster Assignment}}}
                \STATE {Reassign the sample to their nearest center by binary trees compare recursively, gets binary matrix $\langle C^{(t)} \rangle $, where $\langle C_i^{(t)} \rangle \leftarrow \mathcal{F}^k_{min}(\langle D_i^{(t)} \rangle) ( i \in [0,n])$.}
                
                \STATE {}
                \COMMENT{\textbf{\textbf{Centroids Update}}}
                \STATE {A locally calculates $\langle C^{(t)} \rangle^TX_A$.}
                \STATE {B locally calculates $\langle C^{(t)} \rangle^TX_B$.}
                \STATE {A and B securely recompute the cluster centers using $\mu^{(t+1)} =\frac{\langle C^{(t)} \rangle^T \mathcal{X}}{\mathbf{1}_{1\times n}\langle C^{(t)} \rangle}$ by secure sparse matrix multiplication and secret sharing division which is converted to SADD \& SMUL operations.}
                % \STATE {}
                % \COMMENT{\textbf{\textbf{Boundary Check}}}
                \UNTIL{Repeated for a fix number of times or the improvement in one iteration is below a threshold.} 
            \RETURN {Centroids $\langle \mu^{fina} \rangle$}
            \end{algorithmic}
        \end{algorithm}
        \vskip -1in
\subsection{Security Analysis}
    % We present the security proof of the whole framework in Appendix C.
    % \begin{lemma}   
    % Our privacy-preserving (sparse-aware) k-means protocol is a secure instantiation of k-means algorithm against semi-honest adversaries with the existence of ideal MPC functionalities $\mathcal{F}_\mathsf{ESD}$, $\mathcal{F}_\mathsf{min}^k$, $\mathcal{F}_\mathsf{SCU}$, and $\mathcal{F}_\mathsf{CSC}$.
    % \end{lemma}
    %
    % \begin{proof}
    % Please find the proof in Appendix A.
    % \end{proof}
    %
    We present the security proof of our privacy-preserving K-means clustering algorithm in this subsection.
    First, we introduce the security definition of our algorithm against semi-honest adversaries.
    We adopt the security definition from \cite{goldreich2009foundations}, where the security is defined over statistical security parameter $\kappa$ and computational security parameter $\lambda$.
    We also use $\cindist$ to denote computationally indistinguishability by the parameter $\lambda$.
    \begin{definition}
     We say a protocol $\pi$ is a secure instantiation of $f=(f_0, f_1)$ against semi-honest adversaries, if for all sufficiently large $\lambda\in\NN^*$, $x, y\in\bin^*$, there exists two probablistic polynomial-time simulators $(\simulator_0, \simulator_1)$, such that the following holds,
    %  \begin{small}
    % \begin{align*}
    \begin{multline}
        \{(\simulator_0(1^\lambda, x, f_0(x, y)),  f(x,y))\}_{x,y,\lambda} \\  \cindist
        \{(\mathsf{view}^\pi_0(\lambda, x, y), \mathsf{output}^\pi(\lambda, x, y))\}_{x, y,\lambda} 
    \end{multline}
    \begin{multline}
        \{(\simulator_1(1^\lambda, y, f_1(x, y)),  f(x,y))\}_{x,y,\lambda} \\  \cindist
        \{(\mathsf{view}^\pi_1(\lambda, x, y), \mathsf{output}^\pi(\lambda, x, y))\}_{x, y,\lambda}.
    \end{multline}
    % \end{align*}
    % \end{small}
    \end{definition}
    
    \begin{lemma}
    Our privacy-preserving (sparse-aware) k-means protocol is a secure instantiation of k-means algorithm against semi-honest adversaries with the existence of ideal MPC functionalities $\mathcal{F}_\mathsf{ESD}$, $\mathcal{F}_\mathsf{min}^k$, $\mathcal{F}_\mathsf{SCU}$, and $\mathcal{F}_\mathsf{CSC}$.
    \end{lemma}
    
    \begin{proof}
        % The proof of above lemma is straight-forward.
        % 
        Since the executions of both parties are identical, we only show the proof that $P_0$ is corrupted.
        With the existence of the above ideal functionalities, the message received by $P_0$ is:
        
        \medskip
        \noindent
        For each iteration $i$:
        \begin{enumerate}
            \item The output of $\mathcal{F}_\mathsf{ESD}$, denoted as $\set{f_\mathsf{esd}}^{(i)}_0$.
            \item The output of $\mathcal{F}_\mathsf{min}^k$, denoted as $\set{f_\mathsf{min}^k}^{(i)}_0$.
            \item The output of $\mathcal{F}_\mathsf{SCU}$, denoted as $\set{f_\mathsf{scu}}^{(i)}_0$.
            \item The stopping check result $\bin$ where we use $0$ to denote ``continues'', and $1$ to denote ``stop''.
        \end{enumerate}
        The simulation process goes as follows, simulator $\simulator_0$ extracts from the function output: 1) the iteration number $K$, 2) the final results. 
        By definition, the messages 1-3 defined above are indistinguishable from independent randomness. Therefore the simulation of those processes is trivial.
        We let $\simulator_0$ sample independent randomness for steps 1-3, and then for step 4, output $0$ if $i<K$, and $1$ otherwise.
        Clearly, in step 4, the simulator $\simulator_0$ and the real-world protocol at $P_0$ output identical messages.

        Now that if we consider the random initialization points (for all clusters) are public, the training procedure is a deterministic function.
        It allows us to use a simpler definition for semi-honest security:
        \begin{align*}
            \{(\simulator_0(1^\lambda, x, f_0(x, y))\}_{x,y,\lambda}  \cindist
            \{(\mathsf{view}^\pi_0(\lambda, x, y))\}_{x, y,\lambda}
            \\
            \{(\simulator_1(1^\lambda, y, f_1(x, y))\}_{x,y,\lambda}  \cindist
            \{(\mathsf{view}^\pi_1(\lambda, x, y))\}_{x, y,\lambda}.
        \end{align*}

        Therefore, we say the view between the real-world and the ideal world are indistinguishable (steps 1-4).
        The simulation completes.
    \end{proof}
    \noindent
    Also, given that we instantiate all of our functionalities using arithmetic MPC, all constructions of our functionalities $\mathcal{F}_\mathsf{ESD}$, $\mathcal{F}_\mathsf{min}^k$, $\mathcal{F}_\mathsf{SCU}$ satisfy the definition of MPC gates with semi-honest security.

% PRIVACY DISCUSSION

\section{Experiments and Applications}
    % In order to demonstrate the efficiency of our Privacy-Preserving K-means, we focus on answering the following research questions.
    Our experiments intend to answer the following questions:
    \textbf{Q1:} how does our model perform compared with the state-of-the-art privacy-preserving clustering protocol \cite{mohassel2020practical}?
    \textbf{Q2:} How does the division of online and offline affect the performance?
    \textbf{Q3:} How does vectorization affect the performance? 
    \textbf{Q4:} What is the acceleration of our optimized model in sparse scenarios?
    \textbf{Q5:} How about the effectiveness of our framework for real-world multi-party fraud detection applications? 

    \subsection{Experimental Setup}
            
        \nosection{Implementation and Hardware Details}
            We use the C++ programming language to implement our algorithm and run it on a server equipped with a 2.5 GHz Intel Core E5-2682 processor and 188 GB RAM. 
            We consider two network settings. 
            For Q1, we use a Local Area Network (LAN) setting with 0.02ms round-trip latency and 10 Gbps network bandwidth, the same as the comparison model \cite{mohassel2020practical}. 
            For Q2 to Q4, we use the Wide Area Network (WAN) with a 20Mbps maximum throughput and 40ms round-trip latency.
            %
            %
            % To evaluate the performance in different situations, we test our algorithm with different bandwidth in the following. 
            
        \nosection{Hyper-Parameters}
            % HE method
            For \textit{HE method} in our experiments, We choose Okamoto-Uchiyama encryption (OU) \cite{okamoto1998new} which outperforms Pailler over all operations \cite{fang2021large}. 
            % OU \cite{okamoto1998new} because it outperforms Pailler over the operations of encryption, decryption, and homomorphic addition. 
            %
            In this experiment, we set key length as $2,048$ and $\psi$ as a large number with longer than $1,365$ bits (2/3 of the key length) to meet the requirements of both security and efficiency.
            % finite field of SS
            For the \textit{finite field} of SS, we choose integers modulo as $2^{64}$ ($l = 64$), which is computationally beneficial \cite{cramer2018spd}.
            Additionally, we use 20 out of 64 bits to represent the fractional part.
            % multiplication triples generation 矩阵化后的Benchmarks
            For \textit{multiplication triples generation}, we choose OT-based method to implement it \cite{gilboa1999two} and set the security parameter $\kappa = 128, \varphi=3,072$ according to \cite{barker2007nist}.
            With these parameters, the security lifetime of the algorithm can be extended after 2030.

            \begin{table}[t]
                \centering
                % \small
                \setlength{\tabcolsep}{2.5mm}
                {
                    \begin{tabular}{cc|c|c|c|c}
                        \toprule
                        \multicolumn{2}{c|} {Parameters} & \multicolumn{3}{c|} {Ours} & \multirow{2}*{M-Kmeans} \\
                        $n$ & $k$ & Online & Offline & Total Time & ~\\
                        \midrule
                        \multirow{2}*{$10^4$} & 2  & {0.33} & {1.61} & {1.94} & \textbf{1.92} \\
                        ~ & 5 & {0.94} & {4.70} & {\textbf{5.64}} & 5.81 \\
                        \midrule
                        \multirow{2}*{$10^5$} & 2  & {3.12} & {15.19} & {18.31} & \textbf{18.02} \\
                        ~ & 5 & {9.06} & {48.39} & {\textbf{57.45}} & 58.09 \\
                        \bottomrule
                    \end{tabular}  
                }
            
            \caption{Comparison of running time (in minutes) on synthetic data ($t=10$, $l=64$).}
            \label{tab:re1}
            \vskip -0.2in
            \end{table}
      
            \begin{table}[t]
                \centering
                % \small
                \setlength{\tabcolsep}{2.5mm}
                {
                    \begin{tabular}{cc|c|c|c|c}
                        \toprule
                        \multicolumn{2}{c|} {Parameters} & \multicolumn{3}{c|} {Ours} & \multirow{2}*{M-Kmeans} \\
                        $n$ & $k$ & Online & Offline & Total Time & ~\\
                        
                        \midrule
                        \multirow{2}*{$10^4$} & 2  & {1,084} & {3,660} & \textbf{4,744} & 5,118 \\
                        ~ & 5  & {3,156} & {12,900} & {\textbf{16,056}} & 18,632 \\

                        \midrule
                        \multirow{2}*{$10^5$} & 2  & {14,147} & {32,598} & \textbf{46,745} & 47,342 \\
                        %111243
                        ~ & 5  & {33,572} & {131,243} & {\textbf{164,815}} & 192,192 \\
                        \bottomrule
                    \end{tabular}
                }
            % \vskip -0.1in
            \caption{Comparison of communication size (in MB) on synthetic data ($t=10$, $l=64$).}
            \label{tab:re2}
            \vskip -0.3in
            \end{table}
        \nosection{Dataset}
            We conduct experiments on 4 datasets, including 1 real-world dataset and 3 synthetic datasets. 
            %
            % The real-world dataset is used to compare the effect of multi-party modeling. 
            % In addition, this dataset is also used to compare the online-offline setting under practical applications.
            %
            To compare the online-offline setting (in Q1\&Q2) and verify the performance with different parameters (e.g., dimension, sample size, cluster number, and sparse degree in Q3\&Q4), we choose to generate the required data.
            We also take two parties as an example, which is easy to expand to multi-parties.
            % vertically/horizontally
            % 
            For the vertically and horizontally partitioned setting, the only difference is how Eq.~\eqref{Eq:Vectorized ESD} and Eq.~\eqref{Eq:SCU} are calculated, as we have described in Section \ref{sec:lloid}. 
            %
            % there is only one difference between their implementation: the multiplication related to the original data, which is introduced in Section \ref{sec:lloid}.
            %  
            Since computation and communication costs of both settings are almost equivalent, without loss of generality, we vertically split these datasets during experiments. 
            % sparse data
            % The sparse features mainly come from the incomplete user profiles or feature engineering such as one-hot. The whole dataset has 1,236,681 samples
            % The dataset consists of 600 items, all of them with 60 attributes.

        \nosection{Benchmarks}
        The MPC-based K-means protocol in \cite{mohassel2020practical} is the most efficient private K-means scheme among the work that is provable secure.
        %\cite{hegde2021sok}. 
        %
        They proposed an ingenious distance computation protocol for centroids to all samples and built a customized garbled circuit to compute binary secret sharing of the minimum.
        In the remainder of this paper, we term it as M-Kmeans and will compare it in terms of efficiency.
        %
        % For our model and M-kmenas, we focus on the steps of calculating the distance and the overall K-means process.
        %
        % In the following experiments, we will compare with it in terms of efficiency.
        %
        Specifically, we use the publicly available implementation\footnote{https://github.com/osu-crypto/secure-kmean-clustering}. 
        % \cite{ImplSecureKmeanClustering}. 
        %
        As suggested in their paper, we set its computational security parameter $\kappa = 128$, and the bitlength $l = 32$.
        % statistical security parameter $\sigma = 40$,
        
    \subsection{Comparison with M-Kmeans (Q1)}
        \nosection{Methods}
            We compare our model with M-Kmeans in terms of the \textit{run time} and \textit{communication}. 
            Specifically, we show the results of the online phase, offline phase, and entire process of our proposed model and compare them with the total process of M-Kmeans. 
            For the sake of fairness, we follow the paper of M-Kmeans \cite{mohassel2020practical} to run experiments on the synthetic dataset under LAN. % using the same parameters reported in it. 
            The synthetic data is generated from $k \in \{2, 5\}$ clusters with the sample size $n \in \{10^4, 10^5\}$ and feature dimension $d=2$.
            Since the running time and communication increase linearly with the number of iterations, we fix iteration $t=10$.  
        \nosection{Results}
            We report the results of running time and communication cost in Table \ref{tab:re1} and \ref{tab:re2} respectively. 
            From them, we can see that 
            (1) the overall cost (both running time and communication) of our model and M-Kmeans is of the same order of magnitude,
            (2) the online phase of our model is extremely efficient. Our model is about 5x - 6x faster than M-Kmeans in terms of running time and communication volume.
            %
            % The difference between online time and offline time under WAN will be more obvious because the OT-based method for triplets generation in the offline phase is communication consumable.
            %
            Note that if there is a trusted third party that does the offline phase (Beaver's triples generation), the overall efficiency will improve further. 
            %%%
            % In addition, from table \ref{tab:re1} we can find that there is no difference in magnitude between the running time of the online phase and the offline phase. 
            % %
            % % The reasons for this are two-fold. 
            % %
            % It is because that we choose OT-based method for triplets generation, which is communication consumable. The execution of it under LAN setting is much more efficient than WAN setting.
            % %
            % The following experiments for Q2 will show that under the situation of small bandwidth, the advantage of the online phase is clear.
    
        \subsection{Study of Online-Offline Setting (Q2)}
            \nosection{Methods}
            We examine the effect of online-offline framework from two aspects: \textit{running time} and \textit{communication}.
            % run time
            To do it, we run experiments on a synthetic dataset consisting of 1,000 data points with 4 clusters and 2 dimensions in the WAN setting and iteration $t=10$.
            % And the comparison of communication is under the same parameters.
            %
            Recall that secure Lloyd’s iteration consists of three steps, i.e., \textit{secure distance computation} (S1), \textit{secure clusters assignment} (S2), and \textit{secure centroids update} (S3), as described in Section \ref{sec:lloid}. 
            To study the performance of each step, we report results in Figure \ref{fig:on_off_cmp}. 
             
            \nosection{Results}
            As clearly seen from Figure \ref{fig:on_off_cmp}, the data-independent offline phase is time-consuming and communication-consuming, mainly because of the complicated cryptographic operations. 
            We can conclude that after the time-consuming offline phase is ready, the data-dependent online phase is significantly efficient, which is suitable for most practical applications. 

            \begin{figure}[t] 
            \centering
            \subfigure[Running Time]{
                \label{fig:comp_time}
                \includegraphics[height=2.5cm]{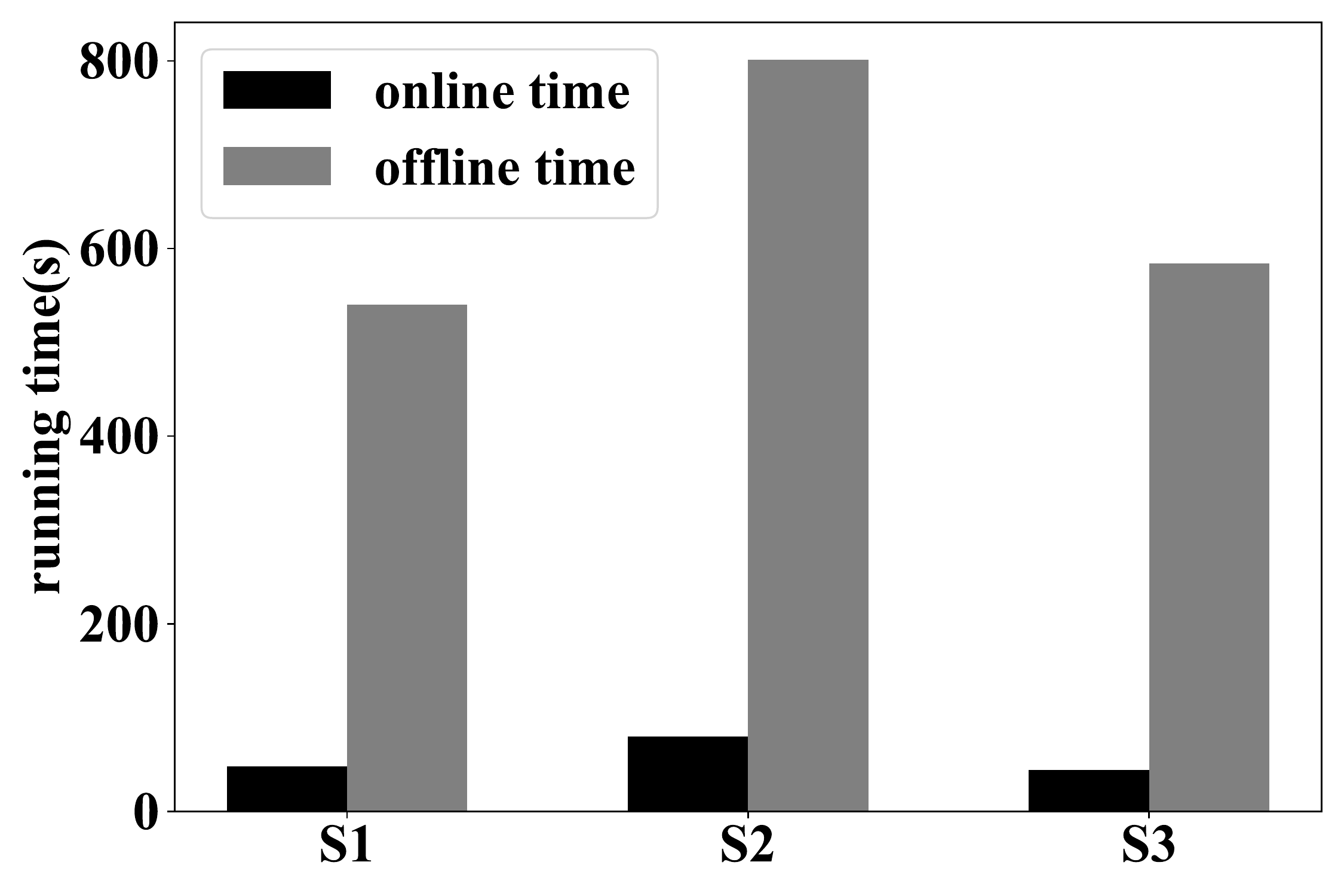}
            }\hspace{0mm}
            \subfigure[Communication]{
                \label{fig:comp_comm}
               ~~~~~ \includegraphics[height=2.5cm]{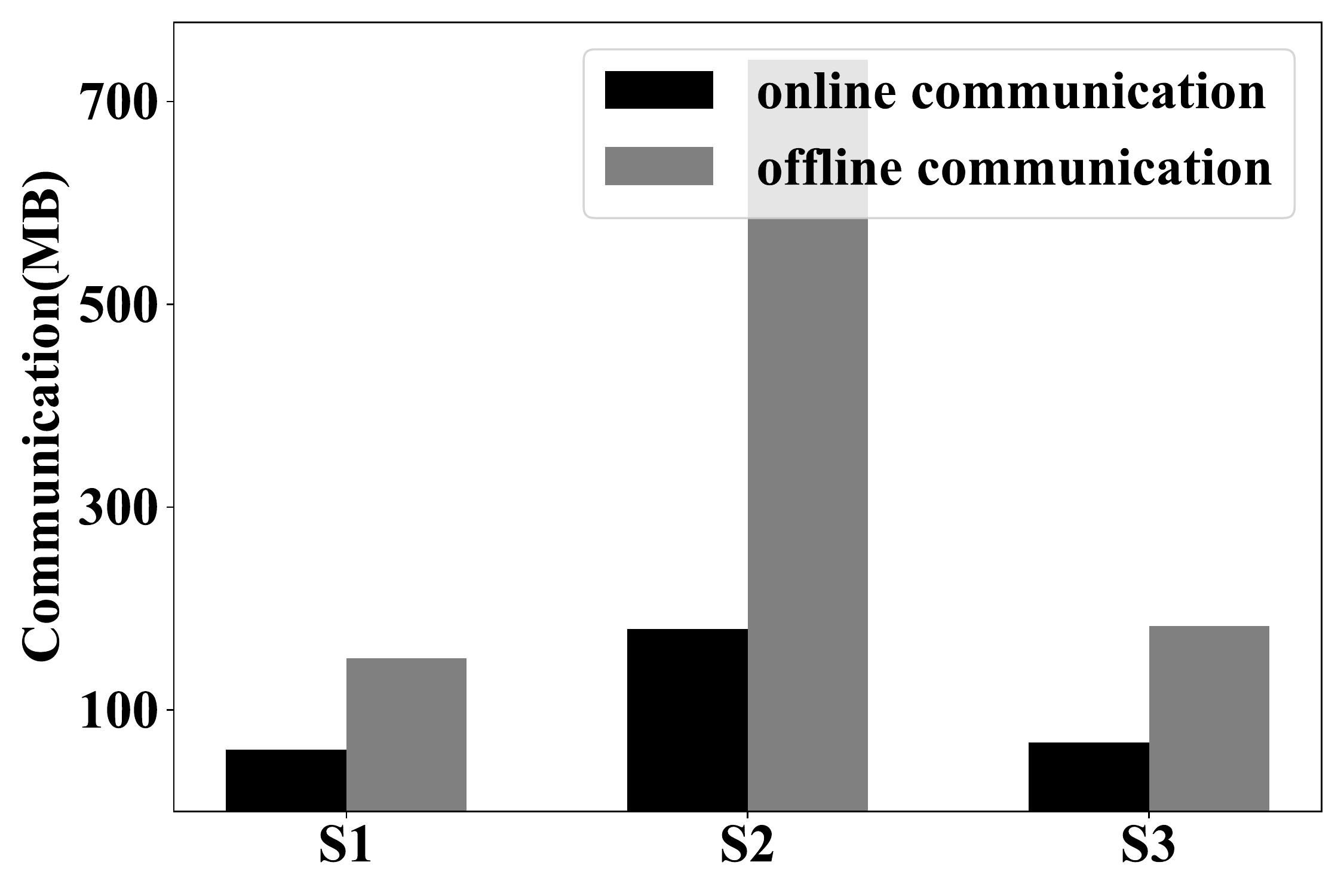}
            }\hspace{0mm}
            % \caption{Comparison of onrline phase and offline phase indifferent steps (sample size $n=1,000$, feature dimension $d=2$, cluster number $k=4$, iteration $t=20$, bit-length $l=64$)}
            % \vskip -0.15in
            \caption{Comparison of online and offline phase in different steps ($n=1,000$, $d=2$, $k=4$, $t=20$, $l=64$).}
            \label{fig:on_off_cmp}
            \vskip -0.1in
            \end{figure}

        \subsection{Experiments for Vectorization (Q3)}
            \nosection{Methods}
            To test the effectiveness of vectorization, we compare running time before and after vectorization under the WAN setting.
            The synthetic data is generated for $4$ clusters with $1,000$ samples and feature dimension $d \in \{2, 4, 6, 8\}$.
            Without loss of generality, we chose the distance calculation step, for example.
            %
            % To analyze the scalability of the protocols to large multi-dimensional datasets, we generate a synthetic collection of large datasets with parameters $n=10^5$, $d \in {2, 4, 8, 16}$ and $K \in {2, 5, 10, 15}$ respectively. 
            % network

            \nosection{Results}
            %  broken line graph
            Figure \ref{fig:vec_cmp} shows the speedup gained from vectorization. 
            We can conclude from it that there are significant improvements for both online phase and offline phases. 
            Compared with the numerical operation, the running time after vectorization increases slower with the growth of the feature dimension.
            In other words, the larger feature dimension, the more improvement of vectorization.
            \begin{figure} 
            \centering
            \subfigure[Online Phase]{
                \label{fig:online_vec}
                \includegraphics[width=0.45\columnwidth]{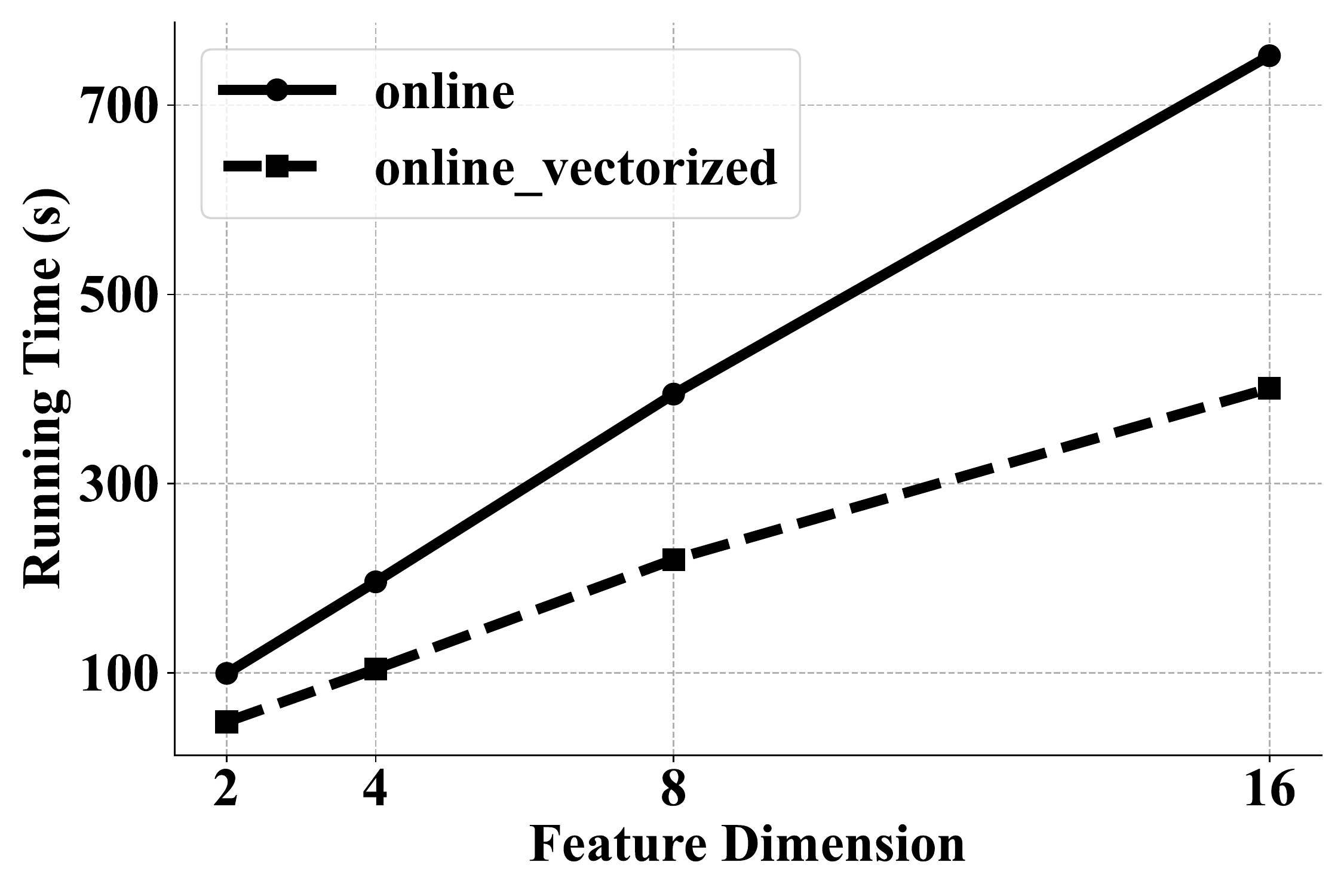}
            }
            \subfigure[Offline Phase]{
                \label{fig:offline_vec}
               ~~~~~ \includegraphics[width=0.45\columnwidth]{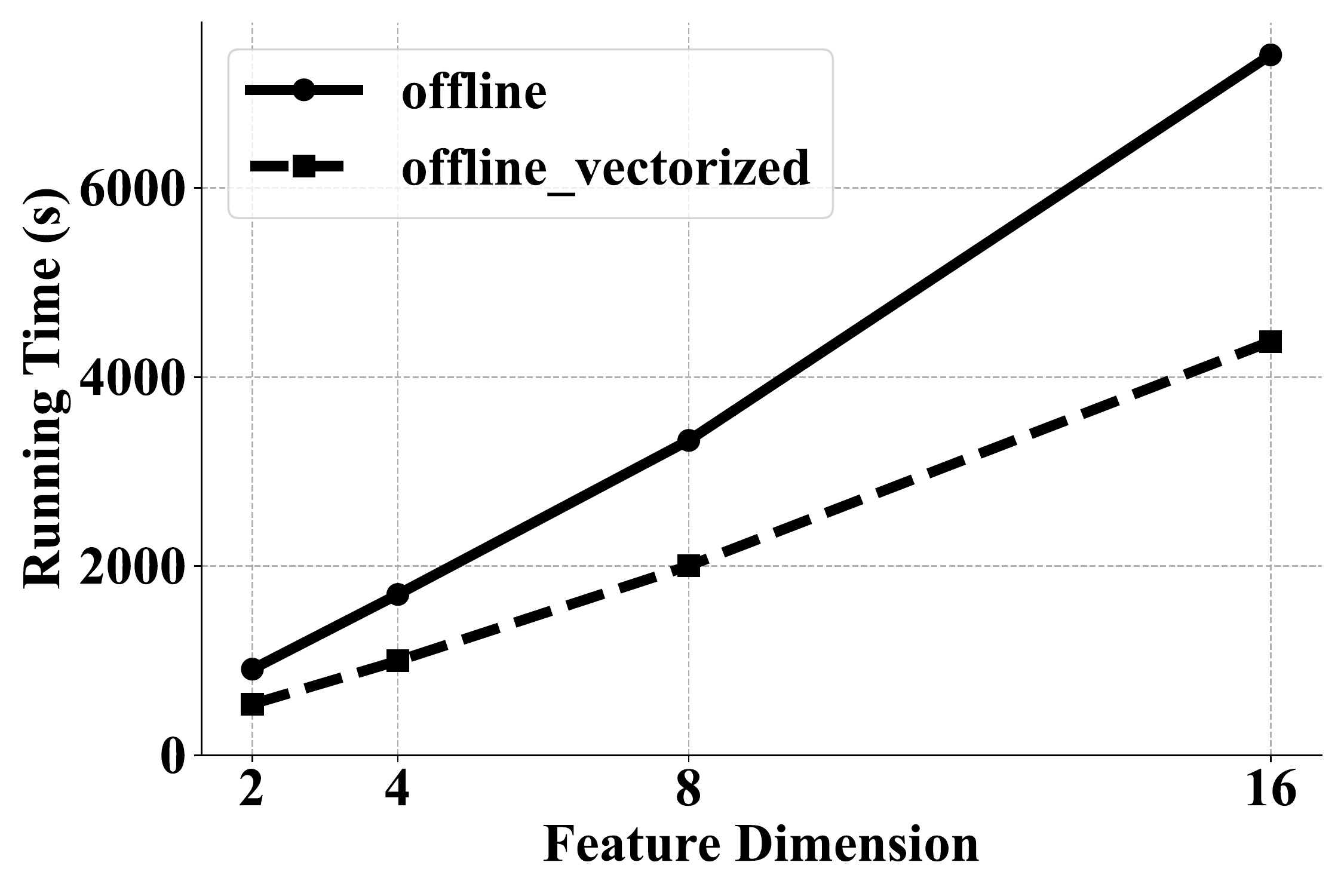}
            }
            \vskip -0.15in
            \caption{Study of vectorization with different feature dimension ($n=10^3$, $k=4$, $t=20$).}
            \label{fig:vec_cmp}
            % \vskip -0.1in
            \end{figure}
            
        \subsection{Experiments for Sparse Optimization (Q4)}
            
            % \noindent\textit{Methods.}
            \nosection{Methods}
            \textbf{(a)} To study our model acceleration in sparse scenarios, we first compare the online running time of our model with and without sparse optimization on a synthetic sparse dataset (with sparse degree 0.2, that is, 20\% of the elements are 0) in WAN setting. As usual, we fix the sample size to $10^6$, cluster number to 2, and iteration to 10. 
            \textbf{(b)} Secondly, we vary the sparsity degree in \{0, 0.5, 0.9, and 0.99\} and sample size from $10^6$ to $5 \times 10^6$ to test the effectiveness of our sparse optimization.
            %
            % Their sparse degree is respectively 0, 0.5, 0.9, and 0.99. And their sample size is varied from $1,000,000$ to $5,000,000$ to simulate the situation under large-scale data. Other parameters are the same.
            %
            We also choose the distance calculation step, for example.
            
            \nosection{Results}    
            \textbf{(a)} We can find from Figure \ref{fig:sparse_dim} that both models scale linearly with feature dimension, but the slope of our model with sparse optimization is smaller than that without sparse optimization. 
            %
            % When the feature dimension is equal to 2, the running time of the base model is slightly faster than the sparse model because the sparse model is computation-consuming but communication-friendly.
            %
            % With the growth of data dimensions, the sparse model quickly showed its advantages, which
            The result proves the scalability of sparse optimization. 
            \textbf{(b)}
            The results shown in Figure \ref{fig:sparse_degree} reveal that our spare optimization can significantly improve efficiency, especially when data sparsity is severe. 
            With the increase in sample size, the improvement becomes more evident, which again proves that our model has good scalability in the large-scale data-sparse scenario. 
                
        %     \begin{figure}
        % 	\center{\includegraphics[width=7cm]  {figure/sparse_dim.png}} 
        %     	\caption{Study of Sparse Optimization in different Feature Dimension (sample size $n=100,000$, cluster number $k=2$, iteration $t=10$, sparse degree 0.2).}
        %     	\label{fig:sparse_dim}
        %     \end{figure}
            
        %     \begin{figure}
        % 	\center{\includegraphics[width=7cm]  {figure/sparse_degree.png}} 
        %     	\caption{Study of Effectiveness of Sparse Optimization (cluster number $k=2$, iteration $t=10$).}
        %     	\label{fig:sparse_degree}
        %     \end{figure}
            %======
            \begin{figure}[t] 
            \centering
            \subfigure[Varying Feature-Dim]{
                \label{fig:sparse_dim}
                \includegraphics[width=0.45\columnwidth]{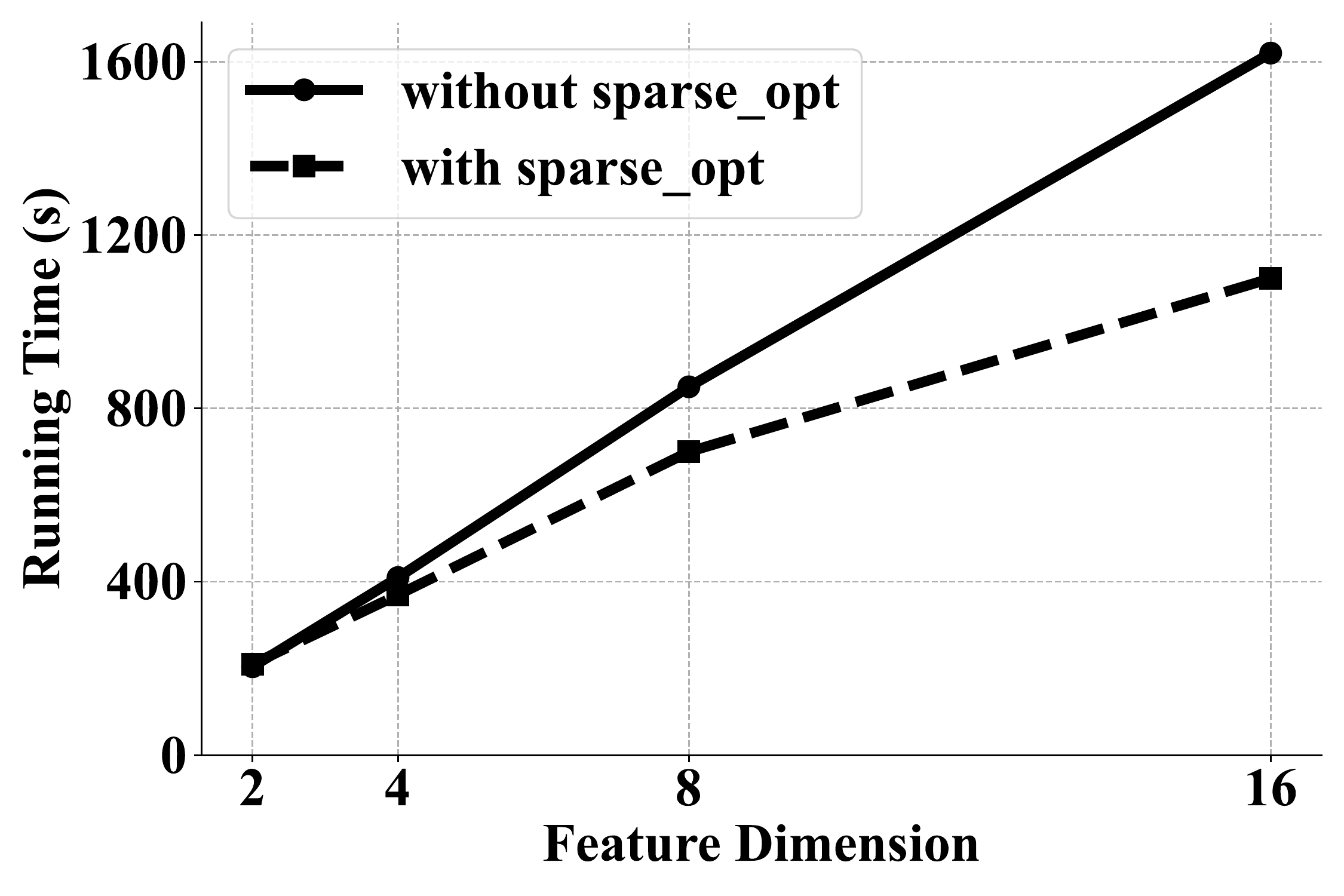}
            }\hspace{0mm}
            \subfigure[Varying Sparse Degree]{
                \label{fig:sparse_degree}
               ~~~~~ \includegraphics[width=0.45\columnwidth]{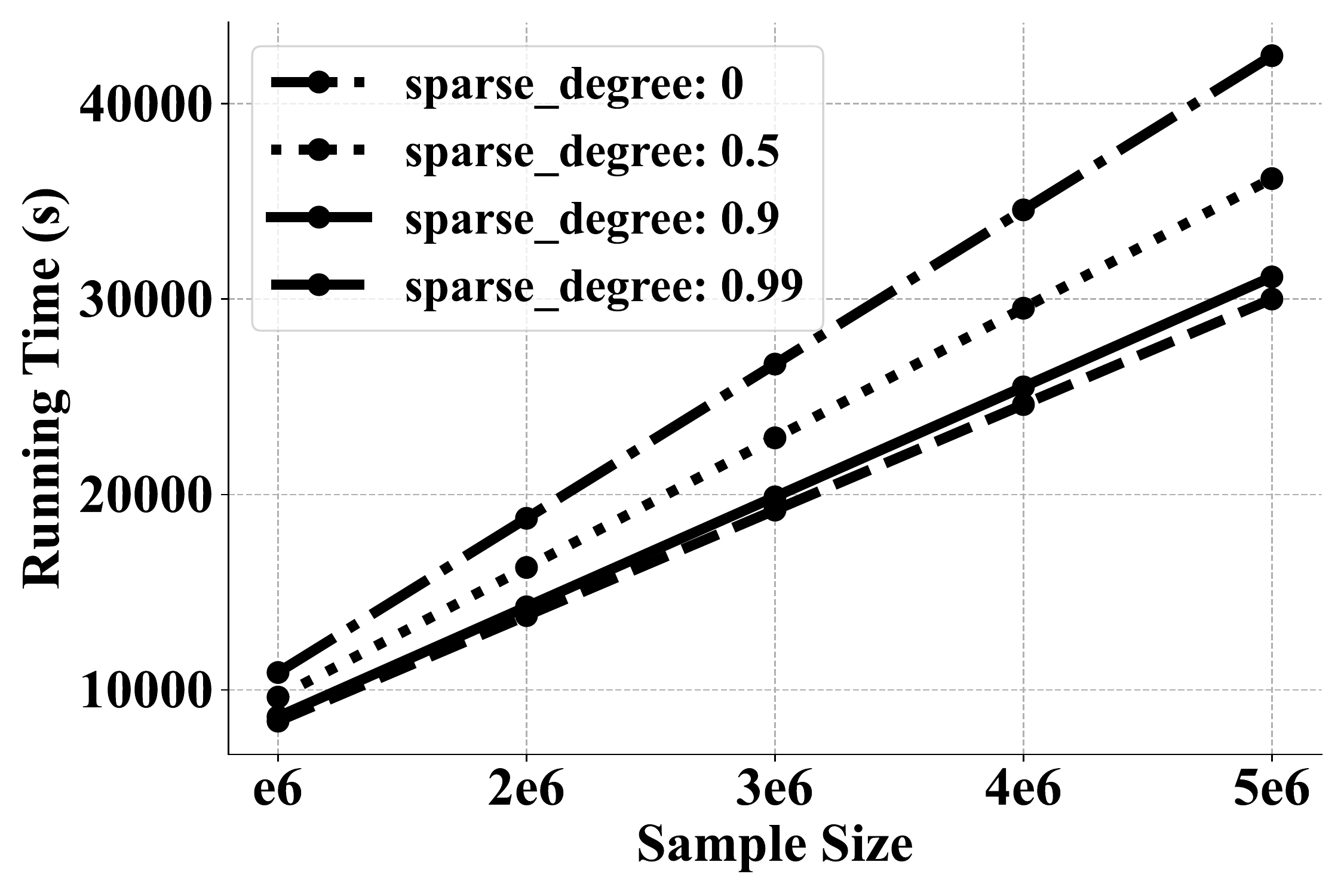}
            }\hspace{0mm}
            % \caption{Study of Effectiveness of Sparse Optimization (cluster number $k=2$, iteration $t=10$).}
            % \vskip -0.15in
            \caption{Study of effectiveness of sparse optimization.}
            \label{fig:sparse_cmp}
            % \vskip -0.1in
            \end{figure}
 
% COST ANALYSIS / Communication Cost Analysis / Computation Cost Analysis

% OU outperforms Pailler over all the operations，，key size (l ∈ {1024, 2048, 3072}).  In practice, l = 2048 turns out to be a good choice that meets the requirements for both security and efficiency.
%  encoding bit

        \subsection{Deployment in Fraud Detection (Q5)}
            \nosection{Scenario}
            % Since more than 90\% of the data in fraud patterns is unlabeled, and the diversity and transience of them, traditional financial companies typically use the K-means algorithm for fraud detection, often in partnership with partner companies.
            % 
            There is a payment company that provides online payment services, whose customers include both individual users and large merchants.
            Users can make online transactions with merchants through it.
            Therefore, rich features, including user features and transaction (context) features in the payment company and merchant, are the key to building intelligent fraud detection models. 
            However, due to the data isolation problem, these data cannot be shared directly. 
            To build a more intelligent fraud detection task, we deploy our work to the above two companies to collaboratively build privacy-preserving K-means clustering.
            
            \nosection{Methods and Data}
            We examine the effectiveness of our framework from two aspects.
            First, we compare its accuracy with M-Kmeans to verify its correctness.
            Second, we compare it with the plaintext K-means using the payment company data only to show the improvements in multi-party modeling.
            The real-world dataset consists of 10,000 data points and 42 dimensions, where the payment company has 18 transaction features (e.g., transaction amount) and partial user features (e.g., user age). In contrast, the merchant has 24 other partial user behavior features (e.g., visiting count).
            Since fraud patterns are diverse and transient, the latest data is used in the application, and 90\% of them have no label.
            %
            % The feature sparsity degree is about 0.9 due to the incomplete user-profiles and feature engineering such as one-hot.
            %
            In order to compare the effect of multi-party modeling on accuracy improvement, the data used in this paper was collected from earlier happened transactions. 
            We can use the Jaccard coefficient to measure the difference between the outliers found by the algorithm and the ground-truth outliers.
            The Jaccard coefficient $(J)$ is defined as:$
            J(R, R^*) = \frac{|R \cap R^*|}{|R \cup R^*|}$, where $R$ is the set of outliers returned by the algorithm, and $R^*$ is the ground-truth.
            Note that by design, $ 0\leq J(R, R^*)\leq 1$. %
            The higher the value, the closer the two sets are.
            
            \nosection{Results}
            We performed 10 runs for each experiment and report the average. 
            The Jaccard coefficient of our framework is $0.86$, while the result of M-Kmeans has the same performance with $0.83$, which proves that our framework can achieve satisfying performance.
            The result of the model using payment company data only is $0.62$, which is much lower than that of joint modeling.
            The results are easy to interpret: more valuable features will naturally improve fraud detection ability.
            Traditionally, the payment company can build the K-means model using its plaintext features only. 
            With our framework, it can build a better model together with the merchant without compromising their private data. 
            % \vskip -1in

\section{Conclusion}

In this paper, 
to solve the efficiency and security problem of the existing secure and privacy-preserving K-means models, 
we propose a novel online-offline vectorized framework to build efficient K-means that provide the fully privacy guarantee.
Especially, for the data-sparse scenario in fraud detection task, we further adopt the large-scale sparse matrix multiplication in K-means model which combines Homomorphic Encryption (HE) and Secret Sharing (SS) to achieves both efficiency and security.
We conduct comprehensive experiments with three synthetic datasets and deploy our model in a real-world fraud detection task. 
Our experimental results show that the execution time of our online phase is 5x faster than the state-of-the-art solution in average while the overall running time is about the same.
In practice, the offline process can be completed in advance.
For the dataset with a certain sparse degree, as the data size increases, the advantages of our solution become more and more obvious.
The results show the efficiency and scalability of our proposed privacy-preserving K-means.

% Finally, we deployed CAESAR into a fraud detection task and conducted experiments on it. In future, we plan to customize CAESAR for more machine learning models and deploy them for more applications.

% Loading bibliography style file
\bibliographystyle{ACM-Reference-Format}
\bibliography{cas-refs}

\end{document}